\documentclass[a4paper]{article}

\usepackage[margin=1.3in]{geometry} 
\usepackage{parskip}
\usepackage{pgf, tikz}
\usepackage{tkz-graph}
\usetikzlibrary{shapes.geometric}
\usetikzlibrary{backgrounds}
\usetikzlibrary{arrows.meta}
\usetikzlibrary{arrows, shapes.arrows, shapes.geometric, shapes.multipart, decorations.pathmorphing, positioning, swigs}

\usepackage[round]{natbib}
\usepackage{amsthm}

\usepackage{amsmath}
\usepackage{mathtools}
\usepackage{centernot}
\usepackage{amssymb}
\usepackage{statmath}
\usepackage{graphicx}
\usepackage{bbm}
\usepackage{tabu}
\usepackage{dsfont}
\usepackage{paralist}
\usepackage{enumitem}
\usepackage{physics}
\usepackage{accents}
\usepackage[ruled,vlined]{algorithm2e}
\usepackage[english]{babel}
\usepackage{amsfonts}
\usepackage{booktabs}
\usepackage{tabu}
\usepackage[T1]{fontenc}
\usepackage{mathrsfs}
\usepackage{multirow}
\usepackage{comment}
\usepackage{stmaryrd}
\usepackage{tikz}
\usetikzlibrary{backgrounds}
\pgfdeclarelayer{bg1}
\pgfdeclarelayer{bg2}
\pgfsetlayers{bg2, bg1, main}
\usepackage{xcolor}
\usetikzlibrary{arrows.meta}
\usepackage{pifont}%

\usepackage{subfig}

\newcounter{thmcount}    
\newtheorem{setting}[thmcount]{Setting}
\newtheorem{theorem}[thmcount]{Theorem}
\newtheorem{remark}[thmcount]{Remark}

\newtheorem{lemma}[thmcount]{Lemma}

\newtheorem{example}[thmcount]{Example}
\newtheorem{excont}[thmcount]{Example}

\newtheorem{proposition}[thmcount]{Proposition}

\newtheorem{definition}[thmcount]{Definition}
\newtheorem{steps}{Algorithm}

\newtheorem{assumption}{Assumption}

\newtheorem{innercustomgeneric}{\customgenericname}
\providecommand{\customgenericname}{}
\newcommand{\newcustomtheorem}[2]{%
  \newenvironment{#1}[1]
  {%
   \renewcommand\customgenericname{#2}%
   \renewcommand\theinnercustomgeneric{##1}%
   \innercustomgeneric
  }
  {\endinnercustomgeneric}
}

\newcustomtheorem{namedsetting}{Setting}

\newcommand\Item[1][]{%
  \ifx\relax#1\relax  \item \else \item[#1] \fi
  \abovedisplayskip=0pt\abovedisplayshortskip=0pt~\vspace*{-\baselineskip}}
\let\emptyset\varnothing

\newcommand{\ind}{\mathbbm{1}} %
\newcommand{\SD}{S^{\diamond}}
\newcommand{\calED}{\mathcal{E}^{\diamond}}

\newcommand{\calEtr}{\mathcal{E}^{\tr}}

\newcommand{\opt}{\operatorname{opt}}

\newcommand{\einv}{\operatorname{e-inv}}
\newcommand{\wrt}{w.r.t.\ }
\renewcommand{\tr}{\operatorname{tr}}

\newcommand{\tst}{\operatorname{tst}}
\renewcommand{\vec}{\operatorname{vec}}

\let\argmax\relax
\let\argmin\relax
\DeclareMathOperator*{\argmax}{argmax}
\DeclareMathOperator*{\argmin}{argmin}

\newcommand{\ci}{\mathrel{\perp\mspace{-10mu}\perp}}
\newcommand\numberthis{\addtocounter{equation}{1}\tag{\theequation}}

\DeclareMathOperator{\EX}{\mathbb{E}}%
\DeclareMathOperator{\VAR}{\mathbb{V}}%
\newcommand{\Cov}{\mathrm{Cov}}

\newcommand{\rvector}[1]{%
  \begingroup\def\arraystretch{0}\begin{bmatrix}
  #1
  \end{bmatrix}\endgroup}

\renewcommand{\P}{\mathbb{P}}

\newcommand{\R}{\mathbb{R}}

\newcommand{\PA}{\mathrm{PA}}

\newcommand\einvce{{e-invariance}}
\newcommand\einvt{{e-invariant}}

    \def \mS {\text{$\mathbf S$}}

 \def \calE {\mathcal E}
 \def \calF {\mathcal F}
 \def \calG {\mathcal G}
 \def \calH {\mathcal H}

 \def \calT {\mathcal T}
 \def \calU {\mathcal U}

 \def \calX {\mathcal X}
 \def \calY {\mathcal Y}

\title{Effect-Invariant Mechanisms for Policy Generalization}

\author{Sorawit Saengkyongam\textsuperscript{1}\thanks{Part of this work was done while SS and JP were at the University of Copenhagen.} ,\,\,Niklas Pfister\textsuperscript{2},\, Predrag Klasnja\textsuperscript{3}, \\ Susan Murphy\textsuperscript{4}, and Jonas Peters\textsuperscript{1}\footnotemark[1]}
\date{\textsuperscript{1}ETH Zürich, \textsuperscript{2}University of Copenhagen, \textsuperscript{3}University of Michigan, \\ \textsuperscript{4}Harvard University \\[2ex]}

\begin{document}

\maketitle

\begin{abstract}
    Policy learning is an important 
    component of many real-world learning systems. A major challenge in policy learning is how to adapt efficiently to unseen environments or tasks. Recently, it has been suggested to exploit invariant conditional distributions to learn models that generalize better to unseen environments. However, assuming invariance of entire conditional distributions (which we call full invariance) may be too strong of an assumption in practice. In this paper, we introduce a relaxation of full invariance called effect-invariance (e-invariance for short) and prove that it is sufficient, under suitable assumptions, for zero-shot policy generalization. We also discuss an extension that exploits
    e-invariance when we have a small sample from the test environment, enabling few-shot policy generalization. Our work does not assume an underlying causal graph or that the data are generated by a structural causal model; instead, we develop testing procedures to test e-invariance directly from data. We present empirical results using simulated data and a mobile health intervention dataset to demonstrate the effectiveness of our approach.
\end{abstract}

\section{Introduction}
When learning models from data, 
we often use these models in scenarios that are assumed to have similar or the same characteristics as the ones generating the training data. 
This holds for 
prediction tasks such as regression and classification 
but also for settings such as contextual bandits or dynamic treatment regimes. 
When we observe different regimes under training, we can hope to exploit this information to construct models that adapt better to an unseen environment (or task). 
Such problems are usually referred to as 
multi-task learning domain adaptation or domain generalization \citep{caruana1997multitask, crammer2008learning, muandet2013domain, wang2022generalizing}; the nomenclature sometimes differs depending on whether one observes labeled and/or unlabeled data in the test domain.
For prediction tasks, it has been suggested to 
learn invariant models by exploiting invariance of the conditional distributions. Under suitable assumptions, such models generalize better to unseen environments if the changes between the environments can be modeled by interventions \citep[e.g.,][]{rojas2018invariant, Magliacane2018, Christiansen2020DG}.
A similar approach has been applied in policy learning \citep{saengkyongam2021invariant}, where one searches for policies that yield an invariant reward distribution.
We refer to the invariance of conditional distributions as `full invariance'. More precisely, given covariates $X_e$ and outcome $Y_e$ from different environments $e \in \calE$, the full invariance assumption posits the existence of a set of covariates $X_e^S$ such that 
\begin{equation}\label{eq:full_invariance}
    \text{for all } e, f \in \calE: Y_e|X_e^S \text{ and } Y_f | X_f^S \text{ are identical}.
\end{equation}
Full invariance, however, may be too strong of an assumption in practice. In prediction tasks, it has been suggested to relax the requirement of full invariance, such as vanishing empirical covariance, and instead use invariance as a form of regularization \citep[e.g.,][]{rothenhausler2021anchor, Jakobsen2020, Arjovsky2019}. This approach comes with theoretical guarantees regarding generalization to bounded interventions, for example, but these results are often limited to restricted classes of models and interventions.

In this paper, we relax the full invariance assumption in a different direction and show how it can be applied to inferring optimal conditional treatments in policy learning. We illustrate our proposed relaxation based on an example. Consider the following
class of structural causal models \citep[SCMs,][]{Pearl2009} indexed by environments $e \in \calE \coloneqq \{1, -1\}$, with the corresponding graph shown in Figure~\ref{fig:ex-intro},
\begin{figure}[t]
    \centering
    \subfloat[ Graphical representation indicating that  full invariance does not hold \label{fig:full_inv}]
    {
    \begin{tikzpicture}[node distance=1.5cm,
    thick, roundnode/.style={circle, draw, inner sep=1pt,minimum size=7mm}, squarenode/.style={rectangle, draw, inner sep=1pt, minimum size=7mm}]
    \node[roundnode] (X2) at (-1.5, 0){$X$};
    \node[roundnode][fill=black!25] (U2) at (0.75,1.5) {$U$};
    \node[squarenode] (E) at (-2.75, 0){$e$};
    \node[roundnode] (A) at (0, 0) {$T$};
        \node[roundnode] (R) at (1.5, 0) {$Y$};
    \draw[-latex] (E) edge (X2);
    \draw[-latex] (X2) edge[bend right=30] (R);
    \draw[-latex] (E) edge[bend right=40] (R);
    \draw[-latex] (X2) -- (A);
    \draw[-latex, thick] (A) -- (R);
    \draw[-latex, thick] (U2) edge (R);
    \draw[-latex, thick] (U2) edge (X2);
\end{tikzpicture}
    }
    \hfil
    \subfloat[ Graphical representation indicating that partial invariance holds \label{fig:par_inv}]
    {
    \begin{tikzpicture}[node distance=1.5cm,
    thick, roundnode/.style={circle, draw, inner sep=1pt,minimum size=7mm}, squarenode/.style={rectangle, draw, inner sep=1pt, minimum size=7mm}]
    \node[roundnode] (X2) at (-1.5, 0){$X$};
    \node[roundnode][fill=black!25] (U2) at (0.75,1.5) {$U$};
    \node[squarenode] (E) at (-2.75, 0){$e$};
    \node[roundnode] (A) at (0, 0) {$T$};
    \node[name=R,shape=swig vsplit,swig hsplit={gap=7pt}, right=.7 of A]{
    \nodepart{left}{$Y_{f}$}
    \nodepart{right}{$Y_{g}$}};
    \draw[shift=(R.center)] plot[mark=+] coordinates{(0,0)};
    \draw[-latex] (E) edge (X2);
    \draw[-latex] (X2) edge[bend right=30] (R);
    \draw[-latex] (E) edge[bend right=30] (R.320);
    \draw[-latex] (X2) -- (A);
    \draw[-latex, thick] (A) -- (R);
    \draw[-latex, thick] (U2) edge[bend left=0] (R.40);
    \draw[-latex, thick] (U2) -- (X2);
\end{tikzpicture}
    }
    \hfill
    \subfloat[ Comparing invariance tests \label{fig:full_vs_par}]
    {
    \begin{tabular}{cc}
    \toprule
    \input{figures/intro_full_vs_partial}
    \end{tabular}
    }
    \caption{
$Y$ is the outcome (or reward), $X$ and $U$ are observed and unobserved context variables, $T$ is the treatment (or action), and $e$ represents different environments. 
In example~\eqref{eq:intro-scm},
the outcome mechanism is generally not invariant (as the environment enters $Y$ directly), see (a). 
This paper introduces a type of partial invariance called e-invariance (Definition~\ref{def:par_inv_set}),
which does hold here, see (b): 
when conditioning on $X$, the treatment effect is invariant across environments. 
The concepts of the paper are applicable even if the data generating process does not allow for a graphical representation. Instead, we propose testing procedures to test for e-invariance. (Bottom) Comparing the test result obtained from one of our proposed e-invariance tests, applied to a sample taken from \eqref{eq:intro-scm}, with the result of the full invariance test. While $X$ does not satisfy the full invariance condition (as in \eqref{eq:full_invariance}), it does satisfy the e-invariance condition (as in \eqref{eq:p-inv_cond1}). 
}
\label{fig:ex-intro}
\end{figure}
\begin{equation}\label{eq:intro-scm}
\mathcal{S}(e):\quad
\begin{cases}
U\coloneqq \epsilon_U \\
X\coloneqq e U + \epsilon_X \\
T\coloneqq \ind(1 + X + \epsilon_T > 0) \\
Y\coloneqq 2e + X + U + T(1 + X) + \epsilon_Y,
\end{cases}
\end{equation}
where $(\epsilon_U, 
\epsilon_X,\epsilon_T,\epsilon_Y)$ are independent standard normal random variables. Here, $Y \in \R$ represents the outcome or reward, $T \in \{0,1\}$ corresponds to the treatment or action, and $X \in \R$ and $U \in \R$ are observed and unobserved covariates, respectively. The mechanism for $T$ can be considered as a fixed policy.
Since the environment has a direct effect on the outcome, there is no subset
satisfying the full invariance condition \eqref{eq:full_invariance}: 
regardless of whether we condition on 
$\emptyset$ or 
$\{X\}$, 
the outcome distribution is not
independent of the environment. Consequently, methods that rely on the full invariance assumption such as the one proposed by \citet{saengkyongam2021invariant} would lead to a vacuous result.

However, the criterion of full invariance is not necessary  when the goal is to learn an optimal policy.
Instead it may suffice to
find models that are \emph{partially invariant}: In the above example, see~\eqref{eq:intro-scm}, the outcome $Y$ can be additively decomposed into two components: one being a function of $U$, $e$, and $\epsilon_Y$, and another being a function of $T$ and $X$. In this case, although the outcome mechanism is not entirely invariant, it contains an invariant component. When conditioning on $X$,
 the effect of the treatment is the same in all environments. More specifically, the conditional average treatment effect does not depend on $e$, that is,
 \begin{equation}
     \forall x \in \calX: \EX^e[Y \mid X = x, T = 1] - \EX^e[Y \mid X, T = 0] = 1 + x. %
 \end{equation}
We say that  $\{X\}$ satisfies \emph{effect-invariance (\einvce{})}. 
This condition suffices that, 
for an unseen test environment, 
we can still infer the optimal treatment among 
policies that only depend on $X$ without having access to the outcome information in the test environment.
In addition, if the environments are heterogeneous enough, such a policy is worst-case optimal. We refer to this setup as zero-shot generalization. We state the class of data generating processes and provide formal results in Section~\ref{sec:partial_inv} and Section~\ref{sec:zero-shot} below.

Moreover, if we can acquire a small sample -- including observations of the outcome -- from the test environment, we would want to optimize the policy using the data from the test environment. Ideally, this optimization also leverages information from training data from other environments to improve the finite sample performance of the learnt policy. We discuss that e-invariant information can be beneficial in such settings.
We refer to this scenario as few-shot generalization and present it as an extension of the zero-shot methodology, in Section~\ref{sec:few-shot}.

While SCMs provide a class of examples satisfying the assumptions of this work,
we do not assume an underlying causal graph or SCM (but instead only require a sequential sampling procedure that ensures that the covariates $X$ causally precede the outcome).
In particular, e-invariance is not read off from a known graph but instead tested from data. 
Figure~\ref{fig:full_vs_par} illustrates the testing result obtained by applying one of the proposed e-invariance tests to a sample from \eqref{eq:intro-scm}, where we also include a comparison with the full-invariance test as proposed in \cite[Method II]{Peters2016jrssb}.

The main contributions of this paper are four fold:
\begin{itemize}
    \item[(1)] Introducing e-invariance: In Section~\ref{sec:partial_inv}, we introduce the concept of e-invariance, which offers a relaxation of the full invariance assumption. An e-invariant set ensures that the conditional treatment effect function remains the same across different environments.
    \item[(2)] Utilizing e-invariance for generalization: Section~\ref{sec:zero-shot} discusses the use of e-invariance in learning policies that provably generalize well to unseen environments.
    We prove two generalization guarantees:
    The proposed method (i) outperforms an optimal context-free policy on new environments and (ii) outperforms any other policy in terms of worst-case performance.  %
    \item[(3)] Methods for testing e-invariance: We propose hypothesis testing procedures, presented in Section~\ref{sec:infer_inv_set}, to test for e-invariance  from data within both linear and nonlinear model classes.
    \item[(4)] Semi-real-world  case study: In Section~\ref{sec:experiments}, we demonstrate the effectiveness of our proposed policy learning methods in the semi-real-world case study of mobile health interventions. An optimal policy based on an e-invariance set is shown to generalize better to new environments than the policy that uses all the context information.
\end{itemize}

\subsection{Further Related Work}
Our work builds upon the existing research that leverages the invariance of conditional distributions (full invariance) for generalization to unseen environments \citep{Schoelkopf2012icml, rojas2018invariant, Magliacane2018, Arjovsky2019, Christiansen2020DG, saengkyongam2021invariant}. Several relaxations of the full invariance have been suggested for the prediction tasks \citep{rothenhausler2021anchor, Jakobsen2020, Arjovsky2019, guo2022two}. 
In reinforcement learning, previous studies have suggested the use of invariance to achieve generalizable policies \citep{zhang2020invariant, sonar2020invariant}, however, they lack theoretical guarantees for generalization. Closely related to our work, \cite{saengkyongam2021invariant} has established the worst-case optimality of invariant policy learning based on the full invariance assumption, which may be too restrictive in practice.

Transportability in causal inference \citep[e.g.,][]{pearl2011transportability, bareinboim2014transportability, subbaswamy2019preventing} addresses the task of identifying invariant distributions based on a known causal graph and structural differences between environments, which can be used to generalize causal findings. However, our approach differs in that we do not assume prior knowledge of the causal graph or structural differences between environments. Furthermore, our methods are applicable
even if the data generating process does not allow for a graphical representation. Instead, we develop testing procedures to obtain invariant information from data. Additionally, methods based on causal graphs typically only capture full invariance information (through the Markov property), whereas our work relaxes the requirement of full invariance for policy learning.

\section{Effect-invariance}\label{sec:partial_inv}
\subsection{Multi-environment policy learning} \label{sec:mepl}
In this work, we consider the problem of multi-environment policy learning (or multi-environment contextual
bandit) \citep[see also][]{dawid2021decision, saengkyongam2021invariant}. 
Given a fixed set of environments
$\calE$, we assume that for each environment $e\in\calE$,
there is a policy learning setup, where the distributions of covariates and outcome may differ between environments. Each of the setups is modelled by a three-step sequential sampling scheme: First, covariates $(X, U)$ are sampled according to a fixed distribution depending on the
environment, then $X$ is revealed to an agent that uses it to
select a treatment $T$ (from a finite set $\mathcal{T}$) according to a policy $\pi$ and, finally, an outcome $Y$ is sampled conditionally on $X$, $U$ and $T$.
Formally, we assume the following setting throughout the paper.

\begin{setting}[Multi-environment policy learning]\label{setting:s1}
  Let $\calE \subset \R$ be a collection of environments, $Y\in\R$ an
  outcome variable, $X \in\mathcal{X}\subseteq\R^d$ observed
  covariates, $U \in\mathcal{U}\subseteq\R^p$ unobserved covariates
  and $T\in\calT = \{1,\dots,k\}$ a treatment. Let $\Delta(\calT)$ denote the probability simplex over the set of treatments $\calT$ and let
  $\Pi\coloneqq\{\pi\mid \pi:\mathcal{X} \rightarrow
  \Delta(\calT)\}$ denote the set of all policies.
  Moreover, for all $e \in \calE$ let  
  $\P^{e}_{X,U}$
  be a distribution
  on $\mathcal{X}\times\mathcal{U}$ and for all $e\in\calE$,
  $x\in\mathcal{X}$, $u\in\mathcal{U}$ and $t\in\calT$ let 
  $\P^{e}_{Y \mid X = x, U = u, T = t}$
  be a distribution 
  on $\R$. 
  Given $e\in\calE$ and $\pi\in\Pi$, this defines a random vector 
  $(Y, X, U, T)$ 
  by 
  $(X,U)\sim
  \P^{e}_{X,U}$, $T\sim\pi(X)$, and $Y \sim 
\P^{e}_{Y \mid X = X, U = U, T = T}$, see Figure~\ref{fig:ex-intro}(a) for an example.
  Correspondingly, $n$ observations $(Y_i, X_i, T_i, e_i, \pi_i)_{i=1}^n$ 
  from this model
  are generated by the following steps. 
    \begin{itemize}
  \item[(1)] Select an environment $e_i\in\calE$ and a
    policy $\pi_i\in\Pi$.
  \item[(2)] Sample covariates $(X_i,U_i)\sim
  \P^{e_i}_{X,U}$.
  \item[(3)] Sample the treatment $T_i\sim\pi_i(X_i)$.
  \item[(4)] Sample the outcome $Y_i \sim 
  \P^{e_i}_{Y \mid X = X_i, U = U_i, T = T_i}$.\footnote{Consequently, the distribution in (4) is indeed the conditional distribution of $Y$, given $X$, $U$, and $T$, justifying the notation.}
\end{itemize}
  The sampling in (2)--(4) is done independently for different $i$. 
In particular, we assume that $e_i$ and $\pi_i$ do not depend on other observables and should be considered fixed. (Our results in Section~\ref{sec:zero-shot} remain valid even if $\pi_i$ depends on previous observations $\{j: j \leq i\}$, see Remark~\ref{rem:1}.)

  Further denote by $\calEtr\subseteq\calE$
  the set of observed environments within the $n$ training observations and for each $e\in\calEtr$ we denote by
  $n_e$ the number of observations from environment $e$. 
  We assume that there exists a product measure $\nu$ such that for all $e\in\calE$
  the joint distribution of $(Y, X, U, T)$ in environment $e$, under policy $\pi$ has density $p^{e,\pi}$ with respect to $\nu$ and
  that $\P^e_X$ has full support on $\mathcal{X}$.
  Next, we define $t_0 \in \calT$ as a baseline treatment, which serves as the reference point for defining the conditional average treatment effect in \eqref{eq:cate_binary}. However, and importantly, our results hold for any choice of $t_0$.
  Finally, we assume that the policies generating the training observations are bounded, i.e., for all $i \in \{1,\dots,n\}, t \in \calT,$ and $x \in \calX$ it holds that $\pi_i(x)(t) > 0$.
\end{setting}

\paragraph{Notation} When writing probabilities and
expectations of the random variables $Y$, $X$, $U$ and $T$ or the corresponding observations,
we use
superscripts to make explicit any possible dependence on the
environment and policy, e.g., $\P^{e,\pi}$ and
$\mathbb{E}^{e,\pi}$. Moreover, by a slight abuse of notation, for a
policy $\pi\in\Pi$ with a density, 
we let $\pi(x)$ denote the density
rather than the distribution; we also use the commonly employed convention
$\pi(t|x) \coloneqq \pi(x)(t)$. Finally, for all $t\in\calT$, we
denote by $\pi_t\in\Pi$ the policy that always selects treatment $t$, that
is, $\pi_t(\cdot|x)=\ind(t=\cdot)$.

\begin{remark}\label{rem:1}
Our results in Section~\ref{sec:zero-shot} remain valid even if $\pi_i$ in Setting~\ref{setting:s1} depends on previous observations. In this case, the sampling step (3) is replaced by $T_i \sim \pi_i(X_i,H_i)$ with $H_i:=\{(X_j,Y_j,T_j): j < i$\}. %
Furthermore, in the \ref{setting:zero-shot} setting in Section~\ref{sec:zero-shot}, we consider 
  $D^{\tr} \coloneqq (y^{\tr}_i,
x^{\tr}_i, t^{\tr}_i, \pi^{\tr}_i(\cdot|\cdot,h_i), e^{\tr}_i)_{i=1}^n$, 
where
$(y^{\tr}_1,
x^{\tr}_1, t^{\tr}_1), \ldots,
(y^{\tr}_n,
x^{\tr}_n, t^{\tr}_n)
$ are (jointly independent) realizations from 
$
Q^{\tr}_1 \coloneqq \P^{e_1,\pi_1(\cdot|\cdot)}_{X, Y, T}, Q^{\tr}_2 \coloneqq \P^{e_2,\pi_2(\cdot|\cdot, h_2)}_{X, Y, T}, \ldots, 
Q^{\tr}_n \coloneqq \P^{e_n,\pi_n(\cdot|\cdot, h_n)}_{X, Y, T}$ respectively, 
with 
$h_i := \{(y^{\tr}_j,
x^{\tr}_j, t^{\tr}_j): j < i\}$ for $i \geq 2$;
in Appendix~\ref{proof:prop:gen_policy_ident}, 
we replace 
$\pi_i$ by
$\pi_i(\cdot|\cdot, h_i)$.
\end{remark}

\subsection{Invariant treatment effects}
The concept of invariance has been connected to causality
\citep{Haavelmo1944, Pearl2009, Schoelkopf2012icml}
and it has been suggested to use it for 
causal discovery 
\citep{Peters2016jrssb, Pfister2018jasa, HeinzeDeml2017} 
or distribution generalization
\citep{rojas2018invariant, rothenhausler2021anchor, Magliacane2018}.
In our setting, the standard notion of invariance 
would correspond to the invariance in the outcome mechanism \citep{saengkyongam2021invariant}.
In practice, this notion may be too strong. E.g., it 
does not hold if the environment directly influences the outcome (see Figure~\ref{fig:ex-intro} for an example). In what follows, we introduce the notion of (treatment) effect-invariance,
which relaxes the standard invariance condition.

To this end, we recall the notion of
the conditional
average treatment effect (CATE) under different environments
$e \in \calE$. 
The CATE
in environment $e\in\calE$ for a subset of covariates
$S\subseteq\{1,\ldots,d\}$ is defined for all $x\in\calX^S$
and $t\in\calT$ as
\begin{equation}\label{eq:cate_binary}
    \tau^S_e(x, t) \coloneqq \EX^{e, \pi_t}[Y \mid X^S=x] - \EX^{e, \pi_{t_0}}[Y \mid X^S=x].
\end{equation}
When $S = \{1,\dots,d\}$, we simply denote $\tau^S_e$ by $\tau_e$. In Setting~\ref{setting:s1}, the CATE functions, as defined in
\eqref{eq:cate_binary},
may differ substantially from one environment to another. But even then, there may exist a subset $S \subseteq \{1,\dots,d\}$ such that the CATE functions do not change
across environments. 
In this work, we exploit the existence of such sets, which we call \emph{e-invariant} (for effect-invariant).\footnote{As an alternative to e-invariance, one could define \emph{argmax-invariance} by requiring that $\forall e_1,e_2 \in \calED: \argmax_{t\in\calT} \tau^S_{e_1}(\cdot, t) \equiv \argmax_{t\in\calT} \tau^S_{e_2}(\cdot, t)$. 
 A similar notion called `invariant action prediction' has been introduced by \citet{sonar2020invariant}.
This condition would ensure that 
the optimal treatment is robust
with respect to changes in the environment (even though the treatment effect may not be). 
E-invariance implies argmax-invariance but the latter condition is not sufficient to show generalization properties that we develop in Section~\ref{sec:zero-shot}.}

\begin{definition}[Effect-invariant sets]\label{def:par_inv_set}
Assume Setting~\ref{setting:s1}. A subset $S \subseteq \{1,\dots,d\}$ is said to be \emph{effect-invariant} with respect to a set of environments $\calED \subseteq \calE$ (\emph{\einvt{}} \wrt $\calED$ for short) if the following holds
\begin{equation}\label{eq:p-inv_cond1}
    \forall e_1,e_2 \in \calED: \tau^S_{e_1} \equiv \tau^S_{e_2}.
\end{equation}
For any $\calED \subseteq \calE$, we denote by $\mS^{\einv}_{\calED}$ the 
collection of all e-invariant sets \wrt $\calED$. 
\end{definition}
The above definition does not depend on the choice of $t_0$ in Setting~\ref{setting:s1}: if  condition \eqref{eq:p-inv_cond1} holds for one $t_0$, it holds for all $t_0 \in \calT$.
In this work, we focus on discrete treatments but, in principle, one could consider the continuous case by defining the CATE function as $(x, t) \mapsto \pdv{}{t}\EX^{e,\pi^t}[Y \mid X^s =x]$ and define the effect-invariance analogously to \eqref{eq:p-inv_cond1}.

We now provide a characterization for e-invariance 
based on the outcome mechanism. 
\begin{proposition}\label{prop:p-inv_eqv}
Assume Setting~\ref{setting:s1}. A subset $S 
\subseteq \{1,\dots,d\}$ is e-invariant \wrt $\calED$ if and only if there exists a pair of 
functions $\psi_S: \calX^S \times \calT \rightarrow \R$ and $\nu_S: \calX^S \times \calE \rightarrow \R$ such that
\begin{equation}\label{eq:p-inv_cond2}
\forall e \in \calED, \forall x \in \calX^S, \forall t \in \calT: \EX^{e, \pi_t}[Y | X^S = x] = \psi_S(x, t) + \nu_S(x, e),
\end{equation}
and $\psi_S(\cdot ,t_0) \equiv 0$.
In particular, 
we have for all $e \in \calED$ that $\psi_S \equiv \tau^S_e$.
\end{proposition}
\begin{proof}
  See Appendix~\ref{proof:prop:p-inv_eqv}.
\end{proof}
The two equivalent conditions \eqref{eq:p-inv_cond1} and \eqref{eq:p-inv_cond2} provide two different viewpoints on e-invariant sets. The former shows that, when conditioning on an e-invariant set $S$, the CATE functions are invariant across environments, while the latter ensures that part of the conditional expected outcome $\EX^{e, \pi_t}[Y \mid X^{S}]$ remains invariant across environments. In particular, the conditional expected outcome $\EX^{e, \pi_t}[Y \mid X^{S}]$ can be additively decomposed into a fixed effect-modification term ($\psi_{S}$) that depends on the treatment and an environment-varying main-effect term $(\nu_{S})$ that does not depend on the treatment. Here, the additivity stems from the definition of the CATE; 
different causal contrasts correspond to other forms of decomposition.

Most of the results in the remaining sections of our work rely on the existence of an e-invariant set. We therefore make this assumption explicit. 
\begin{assumption}\label{assm:exist_inv_set}
In Setting~\ref{setting:s1}, there exists a subset $S \subseteq \{1,\dots, d\}$ such that $S$ is e-invariant \wrt $\calE$. 
\end{assumption}
The subsequent section connects Assumption~\ref{assm:exist_inv_set} to
a class of structural causal models \citep{Pearl2009, Bongers2021, dawid2021decision, saengkyongam2021invariant}. 
For such models, proposition~\ref{prop:suffforS} below shows
that Assumption~\ref{assm:exist_inv_set} is satisfied if the outcome mechanism is of a specific form and an independence assumption holds.
Furthermore, using a test for e-invariance, see Section~\ref{sec:infer_inv_set}, Assumption~\ref{assm:exist_inv_set} is testable from data for the observed environments $\calEtr$.
\subsection{Effect-invariance in structural causal models}

Assumption~\ref{assm:exist_inv_set} is satisfied in a restricted class of 
structural causal models (SCMs).
Formally, we consider the following class of SCMs 
inducing the sequential sampling steps (2)--(4) in Setting~\ref{setting:s1}.
\begin{equation}\label{eq:setting2-scm}
\mathcal{S}(e, \pi):\quad
\begin{cases}
U\coloneqq s_e(X, U, \epsilon_U) \\
X\coloneqq h_e(X, U, \epsilon_X)\\
T\coloneqq \ell_{\pi}(X, \epsilon_T) \\
Y\coloneqq f(X^{\PA_{f,X}}, U^{\PA_{f,U}}, T) + g_e(X, U, \epsilon_Y),
\end{cases}
\end{equation}
where $(U, X, T, Y)\in\mathcal{X}\times\calU\times\calT\times\mathbb{R}$, $(\epsilon_U, 
\epsilon_X,\epsilon_T,\epsilon_Y)$ are jointly independent noise variables, $(s_e, h_e, g_e)_{e \in
\calE}$, $f$ and $\ell_\pi$ are measurable functions such that, for all $x\in\calX$, $\ell_{\pi}(x,
\epsilon_T)$ is a random variable on $\calT$ with distribution $\pi(x)$, and $\PA_{f,X} \subseteq \{1,\dots,d\}$ and $\PA_{f,U} \subseteq \{1,\dots,p\}$. We call $\PA_{f,X}$ and $\PA_{f,U}$ the observed and unobserved policy-relevant parents, respectively.

To determine whether e-invariance holds, it is
helpful to distinguish between the parents of $Y$ that enter $f$ (these are relevant to determine optimal policies) and those parents of $Y$ that enter into $g_e$. For building intuition, we therefore define 
a
graphical representation, which splits $Y$ into two nodes (visually, the graphical representation is similar to SWIGs \citep{richardson2013single}, and we use `tikz-swigs' LaTeX package for drawing the graph; the interpretation, however, is different). 
\begin{definition}[E-invariance graph]
We represent a class of SCMs of the form~\eqref{eq:setting2-scm} by an \emph{e-invariance graph}. 
This graph 
contains,  
as usually done when representing SCMs graphically,
a directed edge from variables on the right-hand side of assignments to variables on the left-hand side, 
but
with the exception that $e$ is represented by a square node and the node $Y$ is split into a part for $Y_f$ and a part for $Y_g$; see Example~\ref{ex:example1} and also Figure~\ref{fig:ex-intro}. 
\end{definition}
\begin{example}\label{ex:example1} Consider the following SCMs 
\\
\begin{minipage}{0.49\textwidth}
\begin{equation*}
\mathcal{S}(e, \pi):\quad
\begin{cases}
        & U^1 \coloneqq \epsilon_{U^1}, \quad U^2 \coloneqq \epsilon_{U^2} \\
        &X^3 \coloneqq \gamma^3_{e} U^1 + \epsilon_{X^3} \\
        &X^2 \coloneqq \gamma^2_{e} U^2 + \epsilon_{X^2} \\
        &X^1 \coloneqq X^2 + \gamma^1_{e} U^1 + \epsilon_{X^1} \\
        &T \coloneqq \ell_{\pi}(X^1, X^2, X^3, \epsilon_T) \\
        &Y \coloneqq \underbrace{T (1 + 0.5X^2 + 0.5 U^1)}_{f} + \\ &\underbrace{\mu_e + U^1 + U^2 + X^2 + X^3 + \epsilon_Y}_{g_e},
\end{cases}
\end{equation*}
\hspace{0.1pt}
\end{minipage}
\begin{minipage}{0.49\textwidth}
\qquad
\begin{tikzpicture}[node distance=1.5cm,
    thick, roundnode/.style={circle, draw, inner sep=1pt,minimum size=7mm}, squarenode/.style={rectangle, draw, inner sep=1pt, minimum size=7mm}]
    \node[roundnode] (X1) at (-1.5, 1.5) {$X^1$};
    \node[roundnode] (X2) at (-1.5, 0){$X^2$};
    \node[roundnode] (X3) at (0, 1.5) {$X^3$};
    \node[roundnode][fill=black!25] (U) at (1.5,2) {$U^1$};
    \node[roundnode][fill=black!25] (U2) at (1.5,-1) {$U^2$};
    \node[squarenode] (E) at (-2.75, 0){$e$};
    \node[roundnode] (A) at (0, 0) {$T$};
    \node[name=R,shape=swig vsplit,swig hsplit={gap=7pt}, right=.7 of A]{
    \nodepart{left}{$Y_{f}$}
    \nodepart{right}{$Y_{g}$}};
    \draw[shift=(R.center)] plot[mark=+] coordinates{(0,0)};
    \draw[-latex] (E) edge (X2);
    \draw[-latex] (X2) -- (X1);
    \draw[-latex] (E) -- (X1);
    \draw[-latex] (E) -- (X3);
    \draw[-latex] (X2) edge[bend right=30] (R);
    \draw[-latex] (E) edge[bend right=70] (R.330);
    \draw[-latex] (U) -- (X3);
    \draw[-latex] (X3) edge[bend left=30] (R.40);
    \draw[-latex] (X1) -- (A);
    \draw[-latex] (X2) -- (A);
    \draw[-latex] (X3) -- (A);
    \draw[-latex, thick] (U) edge[bend right=30] (X1);
    \draw[-latex, thick] (A) -- (R);
    \draw[-latex, thick] (U) -- (R.140);
    \draw[-latex, thick] (U2) edge[bend right=20] (R.320);
    \draw[-latex, thick] (U2) edge[bend left=30] (X2);
\end{tikzpicture}
\end{minipage}
\newline
where $\calT = \{0,1\}$, $\epsilon_{U^1}, \epsilon_{U^2}, \epsilon_{X^1}, \epsilon_{X^2}, \epsilon_{X^3}, \epsilon_T, \epsilon_Y$ are 
jointly independent noise variables with mean zero, and $\gamma^1_e, \gamma^2_e, \gamma^3_e, \mu_e$ 
are environment-specific parameters.  
Here, 
$\PA_{f,X} = \{2\}$ and $\PA_{f,U} = \{1\}$ are the policy-relevant parents; the e-invariance graph is shown on the right.  
While in this example, the environment changes the coefficients $\gamma^1_e, \gamma^2_e, \gamma^3_e$ and $\mu_e$, the generality of \eqref{eq:setting2-scm} allows for a change in the noise distributions, too.
\end{example}

Under the class of SCMs \eqref{eq:setting2-scm}, 
the following proposition shows that an e-invariant set exists if the unobserved $U^{\PA_{f,U}}$ and the observed policy-relevant parents 
$X^{\PA_{f,X}}$ are independent, and the environments do not influence $U^{\PA_{f,U}}$. 
\begin{proposition}\label{prop:suffforS}
Assume Setting~\ref{setting:s1} and that the sequential sampling steps (2)--(4) are induced by the 
SCMs in \eqref{eq:setting2-scm}. If (i) for all $e \in \calED$, $U^{\PA_{f,U}} \ci X^{\PA_{f,X}}$ in $\P^{e}_{X,U}$ and (ii) $\P^e_{U^{\PA_{f,U}}}$ are identical across $e \in \calED$, 
we have that 
\begin{equation}
    \PA_{f,X} \text{ is e-invariant \wrt } \calED.
\end{equation}
\end{proposition}
\begin{proof}
See Appendix~\ref{proof:prop:suffforS}. 
\end{proof}
\begin{excont}[Example~\ref{ex:example1} continued]\label{test}
Let $e \in \calE$. In this example, it holds that $U^{\PA_{f,U}} \ci X^{\PA_{f,X}}$ in $\P^{e}_{X,U}$. Therefore, $\PA_{f,X} = \{2\}$ satisfies the e-invariance condition \eqref{eq:p-inv_cond1} by Proposition~\ref{prop:suffforS}. To illustrate this, consider the expected outcome conditioned on $X^2$,
\begin{align*}
    \EX^{e, \pi_t}[Y \mid X^2] &= \EX^{e, \pi_t}[T (1 + 0.5X^2 + 0.5 U^1) + \mu_e + U^1 + U^2 + X^2 + X^3  \mid X^2] \\
    &= \ind(t = t_0)(1 + 0.5 X^2 + 0.5\EX^{e}[U^1 \mid X^2]) + \mu_e + X^2 + \EX^{e}[U^1 + U^2 + X^3 \mid X^2] \\
    &= \underbrace{\ind(t = t_0)(1 + 0.5 X^2)}_{\psi_{\{2\}}(X^2, t)} + \underbrace{\mu_e + X^2 + \EX^{e}[U^2 + X^3 \mid X^2]}_{\nu_{\{2\}}(X^2, e)} \qquad \text{since $U^1 \ci X^2$}.
\end{align*}
Thus, by Proposition~\ref{prop:p-inv_eqv}, $\{2\}$ is e-invariant \wrt $\calE$.
\end{excont}

\section{Zero-shot policy generalization through e-invariance}\label{sec:zero-shot}

In this section, we consider zero-shot generalization (sometimes called 
unsupervised domain adaptation). We aim to find a policy that performs well (in terms of the 
expected outcome or reward) in 
a new test
environment
in which we have access to observations of the covariates but not the outcome.
We formally lay out the setup and objective of zero-shot policy 
generalization and show that a policy that optimally uses information from e-invariant sets achieve
desirable generalization properties.

\begin{namedsetting}{Zero-shot}\label{setting:zero-shot}
Assume Setting~\ref{setting:s1} and 
that
we are given $n\in\mathbb{N}$ training observations $D^{\tr} \coloneqq (Y^{\tr}_i,
X^{\tr}_i, T^{\tr}_i, \pi^{\tr}_i, e^{\tr}_i)_{i=1}^n$ 
from the observed environments $e^{\tr}_i \in \calEtr$.
During test time, we are 
given $m \in\mathbb{N}$ observations  
$D^{\tst}_X \coloneqq (X^{\tst}_i)_{i=1}^m$ from a 
single test environment $e^{\tst} \in \calE$.
We denote by 
$Q^{\tr} \coloneqq 
Q^{\tr}_1 \otimes \ldots \otimes Q^{\tr}_n$, where
$Q^{\tr}_i \coloneqq \P^{e_i,\pi_i}_{X, Y, T}$ 
and $Q^{\tst}_X \coloneqq \P^{e^{\tst}}_X$ 
the distributions of $D^{\tr}$ and $D^{\tst}_X$, respectively. 
\end{namedsetting}

We seek to find a policy that 
generalizes well to the test environment $e^{\tst}$.
As we only have access to the observed covariate distribution $\P^{e^{\tst}}_{X}$ and since there may be multiple potential test environments $e\in\mathcal{E}$ with $\P^{e}_X=\P^{e^{\tst}}_{X}$, we propose to evaluate the performance of a policy $\pi$ based on its expected outcome (relative to a fixed  baseline policy $\pi_{t_0}$ that always chooses $t_0$) in the worst-case scenario across all environments with covariate distribution equal to $\P^{e^{\tst}}_{X}$.
Formally, let $[e^{\tst}] \coloneqq \{e \in \calE \mid \P^e_X = Q^{\tst}_X\}$
be an equivalence class of
environments under which the covariate distribution $\P^e_X$ is the same as $Q^{\tst}_X$.
We then consider the following worst-case objective
\begin{equation}\label{eq:worst_case_obj}
V^{[e^{\tst}]}(\pi) \coloneqq \inf_{e \in [e^{\tst}]} \big( \EX^{e, \pi}[Y] - \EX^{e, \pi_{t_0}}[Y] \big).
\end{equation}
The goal of (population) zero-shot 
generalization 
applied to our setting
is then to find a policy that (i) is identifiable from $Q^{\tr}_i$ (for an arbitrary $1\leq i \leq n$) and $Q^{\tst}_X$ and 
(ii) maximizes the worst-case performance defined in \eqref{eq:worst_case_obj}.

We now introduce a policy $\pi^{\einv}$ that optimally uses information from e-invariant sets and show that 
$\pi^{\einv}$ achieves the aforementioned goal under suitable assumptions.  
To this end, for all $S \in 
\mS^{\einv}_{\calEtr}$ (see Definition~\ref{def:par_inv_set}),
we denote the set of all policies that depend only on $X^{S}$ by $\Pi^S \coloneqq \{ \pi \in \Pi \mid \exists \bar{\pi}: \calX^S \rightarrow \Delta(\calT) \text{ s.t. } \forall x \in \calX\,, \pi(\cdot | x) = \bar{\pi}(\cdot | x^S) \} \subseteq \Pi$. Next, for all $S \in 
\mS^{\einv}_{\calEtr}$,
we define $\Pi^S_{\opt} \subseteq \Pi^S$ to be a set of policies 
such that each $\pi^S \in 
\Pi^S_{\opt}$ satisfies for all $x\in\calX$ and $t \in\calT$ that
\begin{equation}\label{eq:pi_S}
    \pi^S(t | x) > 0 \implies 
    t  \in \argmax_{t^\prime \in \calT} \tfrac{1}{\abs{\calEtr}} \sum_{e \in \calEtr} \tau^S_e(x^S, t^\prime).
\end{equation}
That is, all the mass of
$\pi^S(\cdot|x)$
is distributed on treatments that maximize the treatment effect conditioned on 
$X^S$. Since $\mS^{\einv}_{\calEtr}$ contains only e-invariant sets \wrt $\calEtr$, we also have that 
$ \tfrac{1}{\abs{\calEtr}} \sum_{e \in \calEtr} \tau^S_e \equiv \tau^S_f$ 
for any fixed $f \in \calEtr$ (but for finite samples, we approximate the former).
Finally, we denote by 
\begin{equation*}
\Pi^{\einv}_{\opt} \coloneqq \{\pi \in \Pi \mid \exists S \in \mS^{\einv}_{\calEtr}\ 
\text{s.t.\ } \pi \in \Pi^S_{\opt}\}
\end{equation*}
the collection of all such policies.

We now propose to use a policy from the collection of policies that are optimal among $\Pi^{\einv}_{\opt}$, i.e.,
\begin{equation}\label{eq:opt_inv_policy}
     \argmax_{\pi\in \Pi^{\einv}_{\opt}} \EX^{e^{\tst}, \pi}[Y].
\end{equation}
Although the set~\eqref{eq:opt_inv_policy} depends on the expected value of $Y$ in the test environment, in Proposition~\ref{prop:gen_policy_ident} we show that we can construct a policy, denoted by $\pi^{\einv}$, that 
satisfies 
the argmax property
\eqref{eq:opt_inv_policy} 
and is identifiable from the data available during training (i.e., i.i.d.\ observations from $Q^{\tr}$ and $Q^{\tst}_X$).

In Theorem~\ref{thm:inv_policy}, we then prove generalization properties of an optimal e-invariant policy $\pi^{\einv}$. This generalization result requires the following two assumptions.
\begin{assumption}[Generalizing environments]
  \label{assm:2}
  It holds for all $S \subseteq \{1,\dots,d\}$ that
  \begin{equation}
    \text{$S$ is e-invariant
      \wrt $\calEtr$}
    \implies \text{$S$ is e-invariant \wrt $\calEtr \cup [e^{\tst}]$}.
  \end{equation}
\end{assumption}
    
Assumption~\ref{assm:2} imposes some commonalities between environments which allows a transfer of e-invariance from 
the observed to the test environments. Similar assumptions 
are used when proving guarantees of
other invariance-based learning methods 
(e.g., \cite{rojas2018invariant, Magliacane2018, Christiansen2020DG, pfister2021SR, saengkyongam2021invariant}).
\begin{assumption}[Adversarial environment]
  \label{assm:3}
  There exist $e\in[e^{\tst}]$ and $S \in \mS^{\einv}_{\calE}$ such that for all $x \in \calX$ it holds that
  \begin{equation} \max_{t \in \calT} \tau_e(x, t) = \max_{t \in \calT} \tau_e^S(x^S, t).
  \end{equation}
\end{assumption}
 Assumption~\ref{assm:3} ensures that there exists at least one environment
 that does not benefit from non-e-invariant covariates and facilitates the worst-case optimality 
 result of our proposed optimal e-invariant policy $\pi^{\einv}$. 
 Without Assumption~\ref{assm:3}, relying only on e-invariant covariates can become suboptimal if other (non-e-invariant) covariates are beneficial across all environments. For example, consider Example~\ref{ex:example1} and assume that the coefficients $\gamma^1_e$ and $\gamma^3_e$ in different environments are relatively close, e.g., $\forall e \in \calE: \gamma^1_e, \gamma^3_e \in (0.9, 1)$. In this scenario, 
 $\{X^1, X^3\}$ is not e-invariant. Still, it is preferable to use these variables for policy learning as they
 provide valuable information for predicting $U^1$, which modifies the treatment effect.
In the above setting, Assumption~\ref{assm:3} does not hold; it would be satisfied if there is at least one additional environment $e \in [e^{\tst}]$ where $\gamma^1_e=\gamma^3_e=0$. The reason is that in such an environment the variables $X^1$ and $X^3$ do not offer any relevant information for predicting $U^1$. A similar assumption, known as confounding-removing interventions, is introduced in \citep{Christiansen2020DG} in the prediction setting.

\begin{proposition}[Identifiability]\label{prop:gen_policy_ident}
Assume Setting~\ref{setting:zero-shot} and Assumptions~\ref{assm:exist_inv_set}~and~\ref{assm:2}. Let $e \in \calEtr$ be an arbitrary training environment, for all $S \in 
\mS^{\einv}_{\calEtr}$ let $\pi^S \in \Pi^S_{\opt}$, that is, a policy that satisfies \eqref{eq:pi_S}, and let $S^*$ be a subset such that 
\begin{equation}\label{eq:opt_inv_policy_explicit}
    S^* \in A \coloneqq \argmax_{S \in \mS^{\einv}_{\calEtr}}\EX^{e^{\tst}} \left[ \textstyle\sum_{t\in\calT}\tau^S_{e}(X^S,t) \pi^S(t \mid X) \right].
\end{equation}
Define $\pi^{\einv} \coloneqq \pi^{S^*}$. Then, the following holds: (i) the set $A$ is identifiable from the distributions $Q^{\tr}_i$
(for an arbitrary $1\leq i \leq n$) and $Q^{\tst}_X$ (which makes it possible to choose $S^*$ and $\pi^{S^*}$ during test time) and (ii) $\pi^{\einv}$ is an element in \eqref{eq:opt_inv_policy}.
\end{proposition}
\begin{proof}
See Appendix~\ref{proof:prop:gen_policy_ident}.  
\end{proof}
\begin{theorem}[Generalizability]
  \label{thm:inv_policy}
  Assume Setting~\ref{setting:zero-shot}~and~Assumptions~\ref{assm:exist_inv_set} and~\ref{assm:2}. 
  Let $\pi^{\einv}$ be as defined in Proposition~\ref{prop:gen_policy_ident}.
Then, the two following
  statements hold.
  \begin{enumerate}[label=(\roman*),
    ref={\thetheorem(\roman*)}]
  \item \label{thm:inv_policy_1}
    Let $\pi_t$, as defined in Section~\ref{sec:mepl}, be the policy that always chooses treatment $t \in \calT$.
    We have that
    \begin{equation}
      \max_{t \in \calT}\EX^{e^{\tst}, \pi_t}[Y]
      \leq 
      \EX^{e^{\tst},\pi^{\einv}}[Y]. \label{eq:lower_bound_invariant_policy_NEW}
    \end{equation}
  \item \label{thm:inv_policy_2} Given Assumption~\ref{assm:3}, we have that
    \begin{equation}
      \forall \pi \in \Pi: V^{[e^{\tst}]}(\pi^{\einv}) \geq V^{[e^{\tst}]}(\pi).
     \end{equation}
  \end{enumerate}
\end{theorem}
\begin{proof}
See Appendix~\ref{proof:thm:inv_policy}.
\end{proof}
Theorem~\ref{thm:inv_policy} provides two generalization properties of the policy $\pi^{\einv}$. First,
Theorem~\ref{thm:inv_policy_1} shows that $\pi^{\einv}$ guarantees to outperform, in any (unseen) test environment, an optimal policy that does not use
covariates $X$.
In other words, it is always beneficial to utilize the information from e-invariant sets when generalizing treatment regimes, compared to ignoring the covariates. Second, 
Theorem~\ref{thm:inv_policy_2} shows that $\pi^{\einv}$ maximizes the worst-case performance defined in 
\eqref{eq:worst_case_obj}, that is, it outperforms all other policies when evaluating each policy in the respective worst case environment if Assumption~\ref{assm:3} holds true. 

\subsection{Estimation of $\pi^{\einv}$}\label{sec:off-policy-est}

As shown in Proposition~\ref{prop:gen_policy_ident}, the policy $\pi^{\einv}$ is identifiable from $Q^{\tr}$ and $Q^{\tst}_X$. We now turn to the problem of estimating $\pi^{\einv}$ given data $D^{\tr}$ and $D^{\tst}_X$ of $Q^{\tr}$ and $Q^{\tst}_X$, respectively. For now, assume  we are given the collection $\mS^{\einv}_{\calEtr}$ of all e-invariant sets \wrt $\calEtr$. We discuss how to estimate $\mS^{\einv}_{\calEtr}$ in Section~\ref{sec:infer_inv_set}. 

Proposition~\ref{prop:gen_policy_ident} suggests a plug-in estimator of $\pi^{\einv}$ based on \eqref{eq:opt_inv_policy_explicit}. 
Specifically, the estimate can be obtained as follows.
\begin{enumerate}[label=(\roman*)]
    \item For all $S \in 
\mS^{\einv}_{\calEtr}$,
compute an estimate $\hat{\tau}^S$ for $\tau^S_e$, $e\in\calEtr$, by pooling the data from the training environments (as the $\tau^S_e$'s are equal across environments by effect-invariance). There is a rich literature on estimating CATE from observational data (see
\cite{zhang2021unified} for a survey), one can choose an estimator that is appropriate to a given dataset. Finally, once an estimate $\hat{\tau}^S$ is obtained, we then plug $\hat{\tau}^S$ into \eqref{eq:pi_S} to construct an estimate $\hat{\pi}^S$ for $\pi^S$, that is, $\hat{\pi}^S$ satisfies for all $x\in\calX$ and $t \in\calT$ that
\begin{equation}\label{eq:hat_pi_S}
    \hat{\pi}^S(t | x) > 0 \implies 
    t  \in \argmax_{t^\prime \in \calT} \hat{\tau}^S(x^S, t^\prime).
\end{equation}
We distribute the probabilities equally if there are more than one $t$ satisfying \eqref{eq:hat_pi_S}.
    \item Find an optimal subset $S^*$ among $\mS^{\einv}_{\calEtr}$, see \eqref{eq:opt_inv_policy_explicit}:
    \begin{equation}\label{eq:opt_inv_est}
    \hat{S}^* \in \argmax_{S \in \mS^{\einv}_{\calEtr}} \frac{1}{m} \sum_{i=1}^{m} \left[\sum_{t\in\calT}\hat{\tau}^S((X^{\tst}_i)^S,t) \hat{\pi}^S(t \mid X^{\tst}_i)\right].
    \end{equation}
    If there are multiple $S$ satisfying \eqref{eq:opt_inv_est}, we randomly choose $\hat{S}^*$ among one of them.
    \item Return $\hat{\pi}^{\hat{S}^*}$-- which was already computed in step (1) -- as the estimate of $\pi^{\einv}$.
\end{enumerate}

\section{Inferring e-invariant sets}\label{sec:infer_inv_set}

We now turn to the problem of testing the e-invariance condition \eqref{eq:p-inv_cond1} based on training observations 
$D^{\tr} \coloneqq (Y_i, X_i, T_i, \pi_i, e_i)_{i=1}^n$ from the observed environments $e_i \in \calEtr$. 

Throughout this section, 
we assume a fixed
initial (or training) policy $\pi^{\tr}$, i.e., $\forall i \in \{1,\dots,n\}: \pi^{\tr} = \pi_i$. The initial policy $\pi^{\tr}$ can either be given or estimated from the available data (see, e.g., Algorithm~\ref{algo:dr_test}). Our proposed testing methods remain valid even if the initial policies $(\pi_i)_{i=1}^n$ are different as long as they are both known and independent of all observed quantities\footnote{Specifically, we do not allow for data collected with adaptive algorithms, which we leave for future work, see Section~\ref{sec:conclfuturework}.}. 
Furthermore, we consider discrete environments, $\calEtr = \{1,\dots,\ell\}$, and consider 
a binary treatment variable, $\calT = \{0, 1\}$. One can generalize to a multi-level 
treatment variable by repeating the proposed procedures for each level $1, \dots, k$ with the baseline 
treatment $t_0 = 0$ and combining the test results with a multiple testing correction method.

To begin with, we define for all $S \subseteq \{1,\dots,d\}$ the e-invariance null hypothesis
\begin{equation} \label{eq:p-inv_null}
H^{\tr}_{0, S} : \text{$S$ is e-invariant \wrt $\calEtr$},
\end{equation}
see Definition~\ref{def:par_inv_set}.

In Section~\ref{sec:linear_effect}, we propose a testing procedure under the assumption that, for all $S \in 
\mS^{\einv}_{\calEtr}$, the functions $(\tau^S_e)_{e\in \calEtr}$ can be modelled by linear functions and provide its 
statistical guarantees. In Section~\ref{sec:nonlinear_effect}, we relax the linearity assumption by using a doubly robust pseudo-outcome learner \citep[see, e.g.,][]{kennedy2020optimal}.

\subsection{Linear CATE functions}\label{sec:linear_effect}

One way of creating e-invariance tests is to assume a parametric form of the CATEs. In this section, we rely on the following linearity assumption.
\begin{assumption}[Linear CATEs]\label{assm:linear_effect}
For all $S \in \mS^{\einv}_{\calEtr}$, there exist coefficients $(\gamma^S_t )_{t \in \calT} \in \R^{k \times \abs{S}}$ and intercepts $(\mu^S_t)_{t \in \calT} \in \R^k$ such that
\begin{equation}\label{eq:tau_inv_linear}
   \forall e \in \calE, \forall t \in\calT, \forall x \in \calX^S: \tau^{S}_e(x, t) = \mu^S_t + \gamma^S_t x.
\end{equation}
\end{assumption}
Under Assumption~\ref{assm:linear_effect}, we now present a testing method for the e-invariance hypothesis $H^{\tr}_{0, S}$ for a fixed set $S \subseteq \{1,\dots,d\}$.
Let $u_e \in \{0, 1\}^{1\times\ell}$ and $v_t \in \{0, 1\}$ 
be the one-hot encodings of 
the environment $e \in \calEtr$ and the treatment $ t \in \calT$,
respectively, and let $\alpha \in \R^{1 \times (1 + d)}$, $A \in \R^{\ell \times (1 + d)}$,
$\beta \in \R^{1 \times (1 + {\abs{S}})}$ and $B \in \R^{\ell \times (1 + {\abs{S}})}$ be 
model parameters. For notational convenience, we define $\tilde{X}_i \coloneqq \rvector{1 & X_i}^\top \in \R^{(1+d)\times1}$ and $\tilde{X}^S_i \coloneqq \rvector{1 & X^S_i}^\top \in \R^{(1+\abs{S})\times1}$. We consider the following (potentially misspecified) response model under treatment $t \in \calT$ and environment $e \in \calEtr$
\begin{equation}\label{eq:linear_response_model}
\underbrace{\alpha \tilde{X} + (u_{e} A) \tilde{X}}_{\text{main effect}} + \underbrace{(v_t \beta) \tilde{X}^S}_{\text{treatment effect}} + \underbrace{(v_t u_{e} B) \tilde{X}^S}_{\text{environment $\times$ treatment effect}}.
\end{equation}
In this model, we have that the CATE functions $\tau^{S}_e$ are identical 
across environments $e \in \calEtr$ if and only if $B = 0$. Thus, testing \eqref{eq:p-inv_null} is equivalent to testing the null hypothesis 
$H_0: B = 0$.

The model proposed in \eqref{eq:linear_response_model} is more restrictive than Assumption~\ref{assm:linear_effect} as it additionally requires the main effect to be linear. To avoid this requirement, we propose using a testing methodology that explicitly allows for the misspecification in the main effect, where we employ the centered and weighted estimation method proposed by \cite{boruvka2018assessing}, which uses a Neyman orthogonal score \citep{neyman1959, neyman1979c}.
(A standard approach of weighted least-squares using weights $1/\pi^{\tr}(T_i|X_i)$ may not yield a test with the correct asymptotic level for the null hypothesis $H_0$.)
More precisely, we consider the following steps:
\begin{enumerate}[label=(\roman*)]
    \item Treatment centering: We center the treatment indicators $v_{T_i}$ by an arbitrary fixed policy $\tilde{\pi}$ that depends only on $X^S$ (i.e., $\tilde{\pi} \in \Pi^S$). More precisely, we replace $v_{T_i}$ with $v_{T_i}
    - \tilde{\pi}(1|X_i^S)$. As an example, one could consider a fixed random policy $\tilde{\pi}(t|x) \coloneqq q^t(1-q)^{(1-t)}$ for some $q \in [0, 1]$. 
    \item Weighted least squares: We estimate the model parameters via a weighted least-squares approach. The weights 
    are defined by $W_i \coloneqq \tilde{\pi}(T_i|X_i^S)/\pi^{\tr}(T_i| X_i)$, where $\tilde{\pi}$ is the policy chosen in step (i) and $\pi^{\tr}$ is the initial policy.
\end{enumerate}

The use of the above steps ensures that the estimator for treatment effects remains consistent even if the main effect is misspecified \citep{boruvka2018assessing} and allows us to obtain a test with pointwise asymptotic level, see Proposition~\ref{prop:wald_test}.

Formally, we employ a generalized method of moments estimator. Define
$\zeta_i(\alpha, A, \beta, B) \coloneqq \alpha \tilde{X}_i + (u_{e_i} A) \tilde{X}_i + (v_{T_i} - \tilde{\pi}(1|X_i^S))(
\beta \tilde{X}^S_i + (u_{e_i} B) \tilde{X}^S_i)$ and $\nabla\zeta_i \coloneqq 
\rvector{\pdv{\zeta_i}{\alpha} & \pdv{\zeta_i}{A} & \pdv{\zeta_i}{\beta} & \pdv{\zeta_i}{B}}^\top$.
We then estimate $\hat{\alpha}, \hat{A}, \hat{\beta}, 
\hat{B}$ as the solutions to the estimating equations
\begin{equation}
\label{eq:ee_glm_wald}
    \sum_{i=1}^n G_i(\alpha, A, \beta, B)=0,
\end{equation}
where $G_i(\alpha, A, \beta, B)\coloneqq W_i[Y_i - \zeta_i(\alpha, A, \beta, B)]{\nabla\zeta_i}$.

Under additional regularity conditions (see Appendix~\ref{proof:prop:wald_test}), we have, for a vectorized $B$, that 
$\sqrt{n}(\hat{B} - B) \indist \mathcal{N}(0, \VAR[B])$.

This allows us to construct a hypothesis test for $H_0: B=0$. To this end, we estimate $\VAR[B]$ as follows. First, for all $i\in\{1,\ldots,n\}$ define 
\begin{equation*}
    \widehat{G}_i\coloneqq G_i(\hat{\alpha}, \hat{A}, \hat{\beta}, \hat{B})\in\R^{s+q}
    \quad\text{and}\quad
    \widehat{J}_i\coloneqq J_i(\hat{\alpha}, \hat{A}, \hat{\beta}, \hat{B})\in\R^{(s+q)\times (s+q)},
\end{equation*}
where $J_i$ is the Jacobian of $G_i$ and $s \coloneqq (1+d+\abs{S})+\ell (1 + d)$ and $q \coloneqq \ell (1 + \abs{S})$. Then, the covariance matrix $\VAR[B]$ can be consistently estimated by the lower block diagonal $q \times q$ entry of the matrix 
\begin{equation}\label{eq:covariance_estimator}
    \left(\frac{1}{n}\sum_{i=1}^n \widehat{J}_i\right)^{-1} \left(\frac{1}{n}\sum_{i=1}^n \widehat{G}_i\widehat{G}_i^\top\right) \left(\frac{1}{n}\sum_{i=1}^n \widehat{J}_i\right)^{-\top},
\end{equation}
\cite[Proposition~3.1]{boruvka2018assessing}.

Let us denote the covariance estimator as $\hat{\VAR}$. We can then use the Wald test to test the null 
hypothesis $H_0: B = \mathbf{0}$ using the consistent estimator $\hat{\VAR}$ of $\VAR[B]$ (see, e.g., \citet{boos2013essential}). When both 
$\tilde{\pi}$ and $\pi^{\tr}$ are given, the covariance estimate can be obtained using standard 
implementations (e.g., Huber-White covariance estimator \citep{huber1967under, white1980heteroskedasticity}). 
However, when either $\tilde{\pi}$ or $\pi^{\tr}$ are estimated, one needs to adjust the covariance estimator to incorporate the additional estimation error (see Supplement~C in \cite{boruvka2018assessing}). The full testing procedure is given in Algorithm~\ref{algo:wald_test}.

\begin{steps}[Wald e-invariance test]\label{algo:wald_test}
Given a training sample $D^{\tr}$ of size $n$, a subset $S \subseteq \{1,\dots,d\}$ and a significance level $\alpha \in (0, 1)$.
\begin{enumerate}[label=(\roman*)]
    \item Solve the estimating equation \eqref{eq:ee_glm_wald}, and compute the test statistic 
    $T_n \coloneqq n\hat{B}^\top\hat{\VAR}\hat{B}$.
    \item Return $ \psi^{\text{Wd}}_n(D^{\tr}, S, \alpha) := \ind(T_n > q_{\alpha})$, where $q_{\alpha}$ is the $(1 - \alpha)$-quantile of a chi-squared distribution with $\ell(1+\abs{S})$-degrees of freedom.
\end{enumerate}
\end{steps}
Proposition~\ref{prop:wald_test} shows that the above results carry over to our setting in that the proposed procedure achieves
pointwise asymptotic level for 
testing the e-invariance hypothesis $H^{\tr}_{0, S}$. 
\begin{proposition}
\label{prop:wald_test}
Assume Setting~\ref{setting:s1}~and~Assumption~\ref{assm:linear_effect}. Let $S \subseteq \{1,\dots, d\}$ be a subset of interest, $\alpha \in (0,1)$ be a significance level, and $\psi^{\text{Wd}}_n(D^{\tr}, S, \alpha)$ be the Wald invariance test detailed in Algorithm~\ref{algo:wald_test}. Under some regularity conditions (see Appendix~\ref{proof:prop:wald_test}), it holds that $\psi^{\text{Wd}}_n(D^{\tr}, S, \alpha)$ has pointwise 
asymptotic level for testing the e-invariance hypothesis $H^{\tr}_{0, S}$, that is,
\begin{equation}
    \sup_{\P \in H^{\tr}_{0, S}}\limsup_{n \to \infty} \P(\psi^{\text{Wd}}_n(D^{\tr}, S, \alpha) = 1) \leq \alpha.
\end{equation}
\end{proposition}
\begin{proof}
    The proof follows directly from \cite[Proposition~3.1]{boruvka2018assessing}, see Appendix~\ref{proof:prop:wald_test}.  
\end{proof}

\subsection{Non-linear CATE functions}\label{sec:nonlinear_effect}
This section relaxes the assumption of linear CATEs (Assumption~\ref{assm:linear_effect}) and proposes a 
non-parametric approach for testing the e-invariance hypothesis $H^{\tr}_{0, S}$. 
The key idea is to employ a pseudo-outcome approach to estimate non-linear CATE functions (see \eqref{eq:cate_binary}) and apply a 
conditional mean independence test based on the pseudo-outcome. In particular, we consider the Doubly
Robust (DR) learner due to \citet{kennedy2020optimal}.

For all $e \in \calEtr$, let $\bar{\mu}^e: \calX \times \calT \rightarrow \R$ denote a model of the conditional expected outcome
$\EX^{e}[Y \mid X = \cdot, T=\cdot]$ and $\bar{\pi}$ denote a model of the initial policy $\pi^{\tr}$. 
Assume $t_0 = 0$. We consider, for all $e \in \calEtr, x \in \calX, t \in \calT$ and $y \in \calY$, the function 
\begin{equation}\label{eq:ps_outcome}
    O^e(x, t, y) = \bar{\mu}^e(x, 1)  - \bar{\mu}^e(x, 0) + \frac{\ind(t = 1)(y - \bar{\mu}^e(x, 1))}{\bar{\pi}(1 | x)} - \frac{\ind(t = 0)(y - \bar{\mu}^e(x, 0))}{1 - \bar{\pi}(1 | x)},
\end{equation}
and generate pseudo-outcomes by plugging in the observed data.
The motivation for constructing the above pseudo-outcome is that, under Setting~\ref{setting:s1}, the conditional mean of $O^e(X, T , Y)$ given $X^S$ is equal to the CATE function $\tau^S_e$ if at least one of the models $\bar{\mu}^e$ or $\bar{\pi}$ is correct. 
Formally, we have the following result.

\begin{proposition}\label{prop:dr_test1}
Assume Setting~\ref{setting:s1}. Let $S \subseteq \{1,\dots, d\}, e \in \calEtr, \pi \in \Pi$ and $O^e(\cdot)$ be the pseudo-outcome defined in \eqref{eq:ps_outcome}. Assume $t_0 = 0$.
If for all $x \in \calX$ and $t \in \calT$
$\bar{\mu}^e(x, t) = \EX^{e, \pi^{\tr}}[Y \mid X=x, T=t]$ or for all $x \in \calX$ and $t \in \calT$
$\bar{\pi}(t \mid x) = \pi^{\tr}(t \mid x)$, 
we have for all $x \in \calX^S$ that 
\begin{equation}
    \EX^{e, \pi^{\tr}}[O^{e}(X, T, Y) \mid X^S=x] = \tau^S_e(x,1).
\end{equation}
\end{proposition}
\begin{proof}
    See Appendix~\ref{proof:prop:dr_test1}.
\end{proof}
Under the assumptions of Proposition~\ref{prop:dr_test1}, 
it holds for all $S\subseteq\{1,\ldots,p\}$ that the null hypothesis $H^{\tr}_{0, S}$ is equal to
\begin{equation} \label{eq:dr_null}
    \forall e_1, e_2 \in \calEtr, \forall x \in \calX^S: \EX^{e_1, \pi^{\tr}}[O^{e_1}(X, T, Y) \mid X^S = x] = \EX^{e_2, \pi^{\tr}}[O^{e_2}(X, T, Y) \mid X^S = x].
\end{equation}
We can thus test for e-invariance by using an appropriate conditional mean independence test that has a correct level under the null hypothesis \eqref{eq:dr_null}. For example, one can use the generalised covariance measure\footnote{The generalised covariance measure (GCM) does not directly test for the conditional mean independence. However, it preserves the level guarantees under the conditional mean independence null hypothesis. Specifically, consider a random vector $(A,B,C)$.
It holds that $\EX[A \mid B, C] = \EX[A \mid B] \implies \Cov(A, B \mid C) = 0 \implies \EX[\Cov(A, B \mid C)] = 0$, where the first equality is the conditional mean independence hypothesis and the last equality is the null hypothesis of the GCM test.  
} \citep{Shah2018, scheidegger2022weighted} or the projected covariance measure \citep{lundborg2022projected}.

We therefore propose the following steps to construct a non-parametric test for the e-invariance hypothesis $H^{\tr}_{0, S}$. 

\begin{steps}[DR-learner e-invariance test]\label{algo:dr_test}
Given a training sample $D^{\tr}$ of size $n$, subset of interest $S$, 
significance level $\alpha$ and conditional mean independence test $\phi$. Let $D_1 \subset D^{\tr}$ denote a random sample of $D^{\tr}$, and $D_2 \coloneqq D^{\tr} \setminus D_1$.
\begin{enumerate}[label=(\roman*)]
    \item Fit models $\bar{\mu}^e$ and $\bar{\pi}$ from the data $D_1$.
    \item Construct the pseudo-outcomes
    \begin{equation*}
        O_i^{e_i}(X_i, T_i, Y_i) = \bar{\mu}^{e_i}_1(X_i)  - \bar{\mu}^{e_i}_0(X_i) + T_i \frac{Y_i - \bar{\mu}^{e_1}_1(X_i)}{\bar{\pi}(1 \mid X_i)} - (1 - T_i) \frac{Y_i - \bar{\mu}^{e_i}_0(X_i)}{1 - \bar{\pi}(1 \mid X_i)}.
    \end{equation*}
    for each observation $(X_i,T_i,Y_i)\in D_2$ in $D_2$.
    \item Apply the test $\phi$ 
    on $O_i^{e_i}(X_i, T_i, Y_i)$ and observations in $D_2$ with a significance level $\alpha$ and return the test result.
\end{enumerate}
\end{steps}

\section{Extension: Few-shot policy generalization through e-invariance}\label{sec:few-shot}
In the \ref{setting:zero-shot} setting, the outcome is not observed in the test environment and, as shown in Theorem~\ref{thm:inv_policy}, relying on e-invariant covariates is optimal under certain assumptions.
This is no longer true if, in the test environment, we have access to observations not only of the covariates but also of the corresponding outcomes obtained after using a test policy in the test environment.
We may then want to adapt to the test environment while exploiting the e-invariance information gathered in the training environments. 
In this section, we 
illustrate how our method could be extended to such a setup (called few-shot generalization), where we observe a large number of training observations from the training
environments and a small number of test observations (including the outcome) from the test environment.

\begin{namedsetting}{Few-shot}
Assume Setting~\ref{setting:s1} and that we are given $n\in\mathbb{N}$ training observations $D^{\tr} \coloneqq 
(Y^{\tr}_i, X^{\tr}_i, T^{\tr}_i, \pi^{\tr}_i, e^{\tr}_i)_{i=1}^n$ 
from the observed environments $e^{\tr}_i \in \calEtr$ and $m\in\mathbb{N}$ test observations $D^{\tst} \coloneqq (Y^{\tst}_i, X^{\tst}_i, T^{\tst}_i, \pi^{\tst}_i)_{i=1}^m$ from a test 
environment $e^{\tst} \in \calE$ and assume that $m \ll n$. 
\end{namedsetting}

The goal of few-shot policy 
generalization is to find a policy $\pi \in \Pi$ that maximizes the expected outcome in the test environment $e^{\tst}$
by exploiting the common information shared between the training and test environments. We consider using Assumption~\ref{assm:2} as the commonalities shared between the environments. In what follows, we propose a constrained optimization approach to learn a policy that aims to maximize the expected 
outcome in the test environment while exploiting the e-invariance condition.

An optimal policy $\pi_{\tst}^*$ in the test environment $e^{\tst}$ distributes all its mass on treatments which maximize the CATE in the test environment -- conditioned on the covariates $X$. That is, an 
optimal policy $\pi_{\tst}^*$ satisfies for all $x\in\calX$ and $t \in\calT$ that
\begin{equation}
    \pi_{\tst}^*(t|x) > 0 \implies 
    t  \in \argmax_{t^\prime \in \calT} \tau_{e^{\tst}}(x, t^\prime).
\end{equation}
Therefore, learning an optimal policy $\pi_{\tst}^*$ can be reduced to learning the CATE function $\tau_{e^{\tst}}$ in 
the test environment. 

As mentioned in Section~\ref{sec:off-policy-est}, the problem of learning $\tau_{e^{\tst}}$ from observational data is a well-studied problem. Here, we abstract away from a specific method and assume that we are 
given a function class $\mathcal{H} \subseteq \{\tau\mid\tau:\calX \times \calT \rightarrow \R \}$ and a loss function $\ell: \calY 
\times \calX \times \calT \times \Pi \times \calH \rightarrow \R$ such that 
$\hat{\tau} \in \argmin_{\tau \in 
\calH} \sum_{i=1}^m \ell(Y^{\tst}_i, X^{\tst}_i, T^{\tst}_i, \pi^{\tst}_i, \tau)$ is a consistent estimator of 
$\tau_{e^{\tst}}$ as $m \rightarrow \infty$. 
Now, we propose to leverage Assumption~\ref{assm:2} when estimating $\tau_{e^{\tst}}$ in the test environment. In particular, by Assumption~\ref{assm:2} we have for all $S \in \mS^{\einv}_{\calEtr}$ and for any fixed $e \in \calEtr$ that
\begin{equation}
    \forall x \in \calX^S, \forall t \in \calT: \tau^S_e(x, t) = \EX^{e^{\tst}}[\tau_{e^{\tst}}(X,t) \mid X^S = x].
\end{equation}
Let $S \in \mS^{\einv}_{\calEtr}$ and for all $x \in \calX^S$, $t \in \calT$ and $\tau \in \calH$, we define $h^S(\tau, x,t) \coloneqq \EX^{e^{\tst}}[\tau(X,t) \mid X^S =x]$ and 
$\tau^S_{\tr}(x, t) \coloneqq \tau_{e}^S(x,t)$ (for an arbitrary $e \in \calEtr$).

We then consider the following constrained optimization
\begin{equation}\label{eq:inv_constraint_opt_policy}
    \begin{aligned}
    \hat{\tau}^S \in & \argmin_{\tau} \sum_{i=1}^m \ell(Y^{\tst}_i, X^{\tst}_i, T^{\tst}_i, \pi^{\tst}_i, \tau) \\
    &\text{s.t.}\ \tau \in \calH \text{ and } \tau^S_{\tr}(\cdot, \cdot) \equiv h^S(\tau, \cdot, \cdot).
    \end{aligned}
\end{equation}
If there are multiple $S \in \mS^{\einv}_{\calEtr}$ satisfying e-invariance, 
that is, $|\mS^{\einv}_{\calEtr}|>1$,
one may choose an optimal set $S^*$ as in \eqref{eq:opt_inv_policy_explicit}.

We now impose the following separability assumption on the CATE function $\tau_{e^{\tst}}$, which allows us to find a solution to the optimization problem \eqref{eq:inv_constraint_opt_policy}.

\begin{assumption}[Separability of CATEs]\label{assm:few-shot}
Let $e^{\tst} \in \calE$ be a test environment, $S \in \mS^{\einv}_{\calEtr}$  and $N\coloneqq\{1,\ldots,d\}\setminus 
S$.
 There exist function classes 
$
\calF \subseteq \{\calX^S \times \calT \rightarrow \R\}$ and
$\calG 
\subseteq \{\calX^N \times \calT \rightarrow \R\}$ and  a pair of functions $f \in \calF$ and $g \in \calG$ such that
\begin{equation}\label{eq:tau_decompose}
    \forall x \in \calX, \forall t \in\calT: \tau_{e^{\tst}}(x, t) = f(x^S,t) + g(x^N, t).
\end{equation}
\end{assumption}
Under Assumptions~\ref{assm:2}~and~\ref{assm:few-shot}, there exists $f \in \calF$ and $g \in \calG$ such that for all $x \in \calX^S$ and $t \in \calT$
\begin{align*}
    \tau^S_{\tr}(x, t) &= \EX^{e^{\tst}}[\tau_{e^{\tst}}(X,t)\mid X^S=x]\\&= \EX^{e^{\tst}}[f(X^S,t) + g(X^N, t) \mid X^S = x] \\
    &=f(x,t) + \EX^{e^{\tst}}[g(X^N, t) \mid X^S =x],
\intertext{which is equivalent to}
 & f(x,t) = \tau^S_{\tr}(x, t) - \EX^{e^{\tst}}[g(X^N, t) \mid X^S =x]. \numberthis \label{eq:few-shot_f}
\end{align*}
Combining \eqref{eq:few-shot_f} and \eqref{eq:tau_decompose}, we then have for all $x \in \calX$ and $t \in \calT$ that
\begin{equation}\label{eq:few-shot-constraint2}
    \tau_{e^{\tst}}(x, t) = \tau^S_{\tr}(x^S, t) - \EX^{e^{\tst}}[g(X^N, t) \mid X^S =x^S] + g(x^N, t).
\end{equation}

Instead of optimizing over the function class $\calH$, we now optimize over the function class $\calG$ by replacing $\tau$ in \eqref{eq:inv_constraint_opt_policy} with $\tau_g^S: (x, t) \mapsto \tau^S_{\tr}(x^S, t) - \EX^{e^{\tst}}[g(X^N, t) \mid X^S =x^S] + g(x^N, t)$. More specifically, we consider the unconstrained optimization
\begin{equation}\label{eq:unconstrained_opt}
    \hat{g} \in \argmin_{g \in \calG} \sum_{i=1}^m \ell(Y^{\tst}_i, X^{\tst}_i, T^{\tst}_i, \pi^{\tst}_i, \tau_g^S).
\end{equation}
Then, $\tau^S_{\hat{g}}$ is a solution to the constraint optimization \eqref{eq:inv_constraint_opt_policy}.

In practice, we estimate the conditional expectation $\EX^{e^{\tst}}[g(X^N, t) \mid X^S = \cdot]$ by an estimator $\hat{q}_{g,t}$.
Intuitively, if the function class $\calG$ (see Assumption~\ref{assm:few-shot}) has a lower complexity compared to $\calH$, and $\hat{q}_{g,t}$ has good finite-sample properties,
one may expect an improvement (e.g., $\tau^S_{\hat{g}}$ has a lower  variance)
using this approach over an estimator that does not take into account the training sample. Without additional assumptions on $\calG$, the optimization problem \eqref{eq:unconstrained_opt} requires the computation of $\hat{q}_{g,t}$ at each iteration (since $\hat{q}_{g,t}$ depends on $g$). In Appendix~\ref{app:fewshotlinear}, we present an example to demonstrate that the optimization 
can simplify when imposing an additional assumption, such as linearity.

\section{Experiments}\label{sec:experiments}
This section presents the empirical experiments conducted on both simulated and real-world datasets. Firstly, we demonstrate through simulations that the testing methods introduced in Section~\ref{sec:infer_inv_set} provide level guarantees that hold empirically in finite samples. Secondly, we demonstrate the effectiveness of our e-invariance approach in a 
semi-real-world case study of mobile health interventions, where it outperforms the baselines in terms of generalization to a new environment.

\begin{figure}[!t]
\centering
\includegraphics[width=0.95\textwidth]{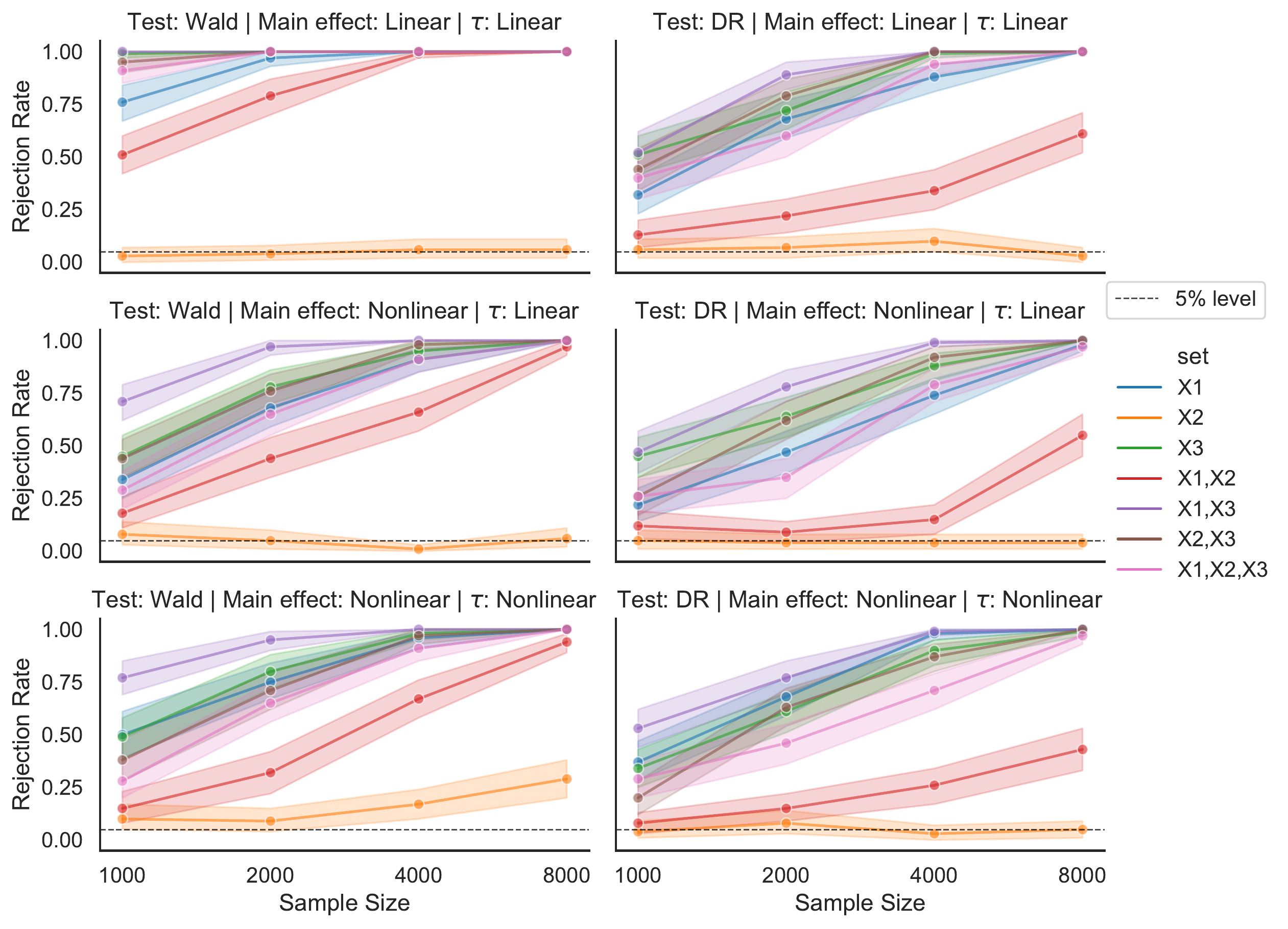}
\caption{
Rejection rates (at the 5\% significance level) of the proposed e-invariance tests from Section~\ref{sec:infer_inv_set} for varying
sample sizes. (Top) the main effect and treatment effect -- see \eqref{eq:linear_response_model} -- are linear, (Middle) the main effect is nonlinear while the treatment effect is linear, (Bottom) both the main effect and treatment effect are nonlinear. In all settings, the DR-learner test achieves the correct level, i.e., the e-invariant set $\{X^2\}$ has a 5\% rejection rate. Similarly, the Wald test correctly rejects the set $\{X^2\}$ at 5\% level in all scenarios except in the bottom scenario due to the violation of Assumption~\ref{assm:linear_effect}. As sample size increases, all other sets are rejected with increasing empirical probability.
}
\label{fig:ex2}
\end{figure}

\subsection{Testing for e-invariance (simulated data)}
We now conduct simulated experiments to validate the e-invariance tests proposed in Section~\ref{sec:infer_inv_set}. We generate datasets of size $n \in \{1000, 2000, 4000, 8000\}$ according to the SCM in Example~\ref{ex:example1} with two training environments $\calEtr = \{0,1\}$.
Each of the noise variables $(\epsilon_{U^1}, \epsilon_{U^2}, \epsilon_{X^1}, \epsilon_{X^2}, \epsilon_{X^3}, \epsilon_T, \epsilon_Y)$ is independently drawn from a standard Gaussian distribution. The environment-specific parameters $(\gamma^1_e, \gamma^2_e, \gamma^3_e, \mu_e)$ are drawn independently from a uniform distribution on $[-3, 3]$. As for the initial 
policy, we consider a policy that depends on the full covariate set $\{X_1, X_2, X_3\}$. More precisely, for all $x \in \calX$, the initial policy $\pi^{\tr}$ selects a treatment according to $\pi^{\tr}(T = 1 \mid X = x) = 1/(1 + e^{-(0.5 + x^1 - 0.5 x^2 + 0.3 x^3)})$. Moreover, we explore a scenario where the assumption of linear main effects in Equation \eqref{eq:linear_response_model} is violated. Specifically, we modify the structural assignment of $Y$ in Example~\ref{ex:example1} as $Y \coloneqq \mu_e + U + X^2 - 0.5 X^2 X^3  + X^3 + T (1 + 0.5X^2 + 0.5 U) + \epsilon_Y$. Lastly, we also consider a setting where the treatment effect itself is nonlinear. In this case, the structural assignment for $Y$ is defined as $Y \coloneqq \mu_e + U + X^2 - 0.5 X^2 X^3  + X^3 + T (1 + 0.5(X^2)^2 + 0.5(X^2)^3 + 0.5 U) + \epsilon_Y$.

We then conduct the Wald and DR-learner e-invariance tests (Wald test and DR test for short, respectively) for all candidate subsets according to Algorithm~\ref{algo:wald_test} and Algorithm~\ref{algo:dr_test}, where we assume that the initial policy $\pi^{\tr}$ is given. For the DR test, we estimate the conditional mean function ($\bar{\mu}^e$) with a random forest 
\citep{breiman2001random} and use the weighted generalised covariance measure \citep{Shah2018, scheidegger2022weighted} as the final test $\phi$ in Algorithm~\ref{algo:dr_test}.

Figure~\ref{fig:ex2} reports the rejection rates at the 5\% significance level for each candidate set under various settings. Recall that in Example~\ref{ex:example1},  $\{X^2\}$ is the only e-invariant set.
The results indicate that, for finite sample sizes, 
both of the proposed methods hold the correct level at 5\% in all settings (the rejection rates for the e-invariant set $\{X^2\}$ are approximately 5\% in all settings) except in the bottom left setting: here, the linear CATEs assumption (Assumption~\ref{assm:linear_effect})
is violated and
the Wald test fails to maintain the correct level.
When the linear main effect and treatment effect assumptions in \eqref{eq:linear_response_model} are specified correctly (top row), the Wald test shows superior performance compared to the DR test (that is, the Wald test rejects the non-e-invariant sets more often). When the linear main effect assumption is violated (middle row), the Wald test remains valid but the power of the test drops significantly. The Wald test, nonetheless, slightly performs better than the DR test in terms of test power in this setting.

\subsection{A case study using HeartSteps V1 dataset}

We apply our proposed approach to the study of a mobile health intervention for promoting physical activity called HeartSteps V1 \citep{klasnja2019efficacy}. HeartSteps V1 was a 42-day micro-randomized trial with 37 adults that aimed to optimize the effectiveness of two intervention components for promoting physical activity. One of the interventions was contextual-aware activity suggestions, delivered as push notifications, which aimed to encourage short bouts of walking throughout the day. Each participant was equipped with a wearable tracker that linked to the mobile application, which gathered sensor data and contextual information about the user. This information was used to tailor the content of activity suggestions that users received and to determine whether the user was available to receive an activity suggestion (e.g., if the sensor data indicated that the user was currently walking, they would not be sent a suggestion). The application randomized the delivery of activity suggestion up to five times a day at user-selected times spaced approximately 2.5 hours apart. If the contextual information indicated that the person was unavailable for the intervention, no suggestion was sent.

In this paper, we consider users as environments. We filter out users that had zero interactions with the application, resulting in a total of 27 users. For each user $u \in \{1,\dots,27\}$, we have the user's trajectory $(X_{u, 1}, T_{u, 1}, Y_{u, 1}),\dots,(X_{u, \ell_u}, T_{u, \ell_u}, Y_{u, \ell_u})$ of size $\ell_u$ (on average $\ell_u$ is 160), where the covariates $X_{u,i}$ 
are the contextual information about the user at time step $i$, the treatment $T_{u,i} \in \{0,1\}$ is whether to deliver an activity suggestion, and the outcome $Y_{u,i}$ is the log transformation of the 30-minute step count after the decision time. In this analysis, 
we make Assumption~\ref{assm:linear_effect} and consider the following approximation for the conditional mean of $Y_{u,i}$:
\begin{equation}\label{eq:reward_HSV1}
    \alpha_{u}^\top g(X_{u,i}) + \beta_{u}^\top f(X_{u,i}) T_{u,i},
\end{equation}
where $g(X_{u,i})$ is a (known)
baseline feature vector and $f(X_{u,i})$ is a
(known)
feature vector for the treatment effect. 
We allow the main effect to be misspecified.
As for the feature vectors,
we consider the same features (with minor modifications\footnote{We replace the dosage variable with the bucketized decision time variable to account for potential non-linear time dependency.}) as in \cite{liao2020personalized}; the vector $f(X_{u,i})$ contains Decision Bucket (DB) (bucketized decision time), Application Engagement (AE) (indicating how frequently users interact with the application), Location (LC) (indicating whether users are at home, at work or somewhere else)  and Variation Indicator (VI) (the variation level of step count 60 minutes around the current time slot in past 7 days). The baseline vector $g(X_{u,i})$ contains $f(X_{u,i})$ along with the prior 30-minute step count, the previous day's total step count and the current temperature. 

Since, for a given user, the outcome model \eqref{eq:reward_HSV1} does not change over time, we can combine all users' trajectories and obtain the combined observations under multiple environments; that is, we have the dataset $D \coloneqq (X_i, T_i, Y_i, e_i)_{i = 1}^n$, with $n = \sum_{u \in \{1,\dots,27\}} \ell_u$, collected from multiple environments (users) $\calE$, where $e_i \in \calE$ for all $i \in \{1,\dots,n\}$. In particular,
we do not account for potential temporal dependencies that are not captured by the bucketized decision times.
In practice, one may allow for dependence across time in the observations $(X_{u, i}, T_{u, i}, Y_{u, i})_{i=1}^{\ell_u}$ within each user, which we leave for future work, see Section~\ref{sec:conclfuturework}.

\begin{figure}[t]
\centering
\begin{minipage}{0.45\textwidth}
\includegraphics[width=1.0\textwidth]{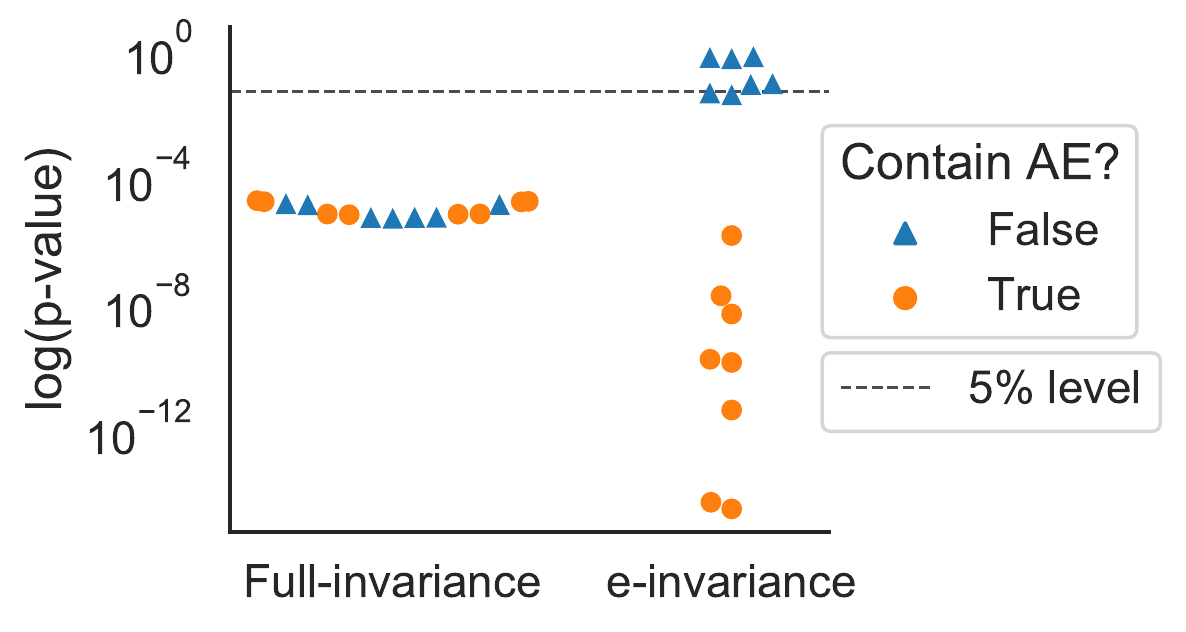}
\end{minipage}
\hspace{20pt}
\begin{minipage}{0.45\textwidth}
\begin{tabular}{cc}
\toprule
\input{figures/invariant_sets}
\end{tabular}
\end{minipage}
\caption{(Left) P-values for all subsets, considering the full- and e-invariance hypotheses. (Right) all five subsets for which we do not reject the e-invariance hypothesis.} 
\label{fig:full_vs_e_test}
\end{figure}

\subsection{Inferring e-invariant sets (HeartSteps V1)}
We begin our analysis on the HeartSteps V1 data by conducting the Wald e-invariance test detailed in Algorithm~\ref{algo:wald_test} to find subsets of the treatment effect feature vector $f(X)$ that satisfy the e-invariant condition \eqref{eq:p-inv_cond1}. As a comparison, we also apply the invariance test proposed in \citet[Method II]{Peters2016jrssb}, which tests for a full-invariance instead of our proposed e-invariance (see Figure~\ref{fig:ex-intro}). 
Figure~\ref{fig:full_vs_e_test}(Left) reports the p-values of all subsets for the full-invariance and e-invariance tests. The p-values for the full-invariance are all below the 5\% level and hence there is no subset that satisfies the full-invariance hypothesis. 
However, we find
several subsets that satisfy the e-invariance condition (those with p-values of the e-invariance hypothesis greater than the 5\% level). Interestingly, all subsets that contain Application Engagement (AE) have p-values close to zero,
suggesting
that AE is a variable that renders the conditional treatment effect unstable between environments if included in the model. We report all the subsets for which we accept the e-invariance hypothesis at the 5\% significance level in Figure~\ref{fig:full_vs_e_test}(Right). 

The above
finding demonstrates that the relaxed notion of invariance that we propose 
can be
beneficial in practice. The full-invariance condition may be too strict in that there is no full-invariant set.  
But if our goal is to learn a generalizable policy, it 
may suffice to test for the weaker notion of e-invariance, which the following section investigates using semi-real data.

\subsection{Zero-shot generalization (augmented HeartSteps V1)}
As the HeartSteps V1 study has been completed, it is not possible to implement and test a proposed policy on a new subject.
In this section, we instead conduct a simulation study using HeartSteps V1 data to illustrate the use of e-invariance for zero-shot generalization, see Section~\ref{sec:zero-shot}. To evaluate the performance of a policy, we consider `leave-one-environment-out' cross validation. Specifically, we first choose $e \in \calE$ as a test environment (user) and split the dataset $D$ into the test set $D^{\tst} \coloneqq \{(X^{\tst}_i, T^{\tst}_i, Y^{\tst}_i, e^{\tst}_i)\}_{i=1}^{n_{\tst}}$ and the training set $D^{\tr} \coloneqq \{(X^{\tr}_i, T^{\tr}_i, Y^{\tr}_i, e^{\tr}_i)\}_{i=1}^{n_{\tr}}$, where $e^{\tst}_i = e$ and $e^{\tr}_i \in \calEtr \coloneqq \calE\setminus\{e\}$ for all $i$. We then conduct the training and testing procedure as follows.

\paragraph{Training phase:} Using the training data $D^{\tr}$, we find all 
sets that are not rejected by the
Wald e-invariance test detailed in Algorithm~\ref{algo:wald_test}. Using the inferred e-invariant sets, we then compute an estimate of $\pi^{\einv}$ as discussed in Section~\ref{sec:off-policy-est}, where we use the R-learner due to \citet{nie2021quasi} as the CATE estimator -- based on the implementation of the \texttt{econml} Python package \citep{econml}. 
As a baseline, we include an optimal policy which utilizes all variables in $f(X)$ (denoted as `full-set'). This baseline is computed by pooling all data from the training environments and fitting the R-learner CATE estimator on the complete covariate set. Additionally, we include a uniformly random policy denoted as `random' as another baseline for comparison.

To illustrate this procedure for $e=1$, i.e., consider the set of training users $\calEtr = \{2,\dots,27\}$. Using the observations from $\calEtr$, we apply Algorithm~\ref{algo:wald_test} to obtain the inferred e-invariant sets $\mS^{\einv}_{\calEtr} = \{\{\text{DB}\}, \{\text{DB}, \text{VI}\}, \{\text{VI}\}, \{\text{DB}, \text{LC}, \text{VI}\}\}$. For each $S \in \mS^{\einv}_{\calEtr}$, we then train a policy $\hat{\pi}^S$ as in \eqref{eq:hat_pi_S} using the R-learner as the CATE estimator and choose an optimal $\hat{S}^*$ as in \eqref{eq:opt_inv_est}. We then use $\pi^{\hat{S}^*}$ as the final estimate of $\pi^{\einv}$.

\paragraph{Testing phase:} 
To perform policy evaluation, we create a semi-real test environment. 
To do so, we follow \citet{liao2020personalized}.
Given a test dataset $D^{\tst}$, the value $V(\pi)$ of a policy $\pi \in \Pi$ is computed by the following procedure.
\begin{enumerate}
    \item[(1)] Fit a regression model \eqref{eq:reward_HSV1} on $D^{\tst}$ 
    \begin{equation}
        Y^{\tst}_i = \alpha_{\tst}^\top g(X^{\tst}_i) + \beta_{\tst}^\top f(X^{\tst}_i) T^{\tst}_{i} + \epsilon_i,
    \end{equation}
    and obtain pairs of covariates and residuals $\{(X^{\tst}_i, \hat{\epsilon}_i)\}_{i=1}^{n_{\tst}}$ and parameters $\hat{\alpha}_{\tst}$ and $\hat{\beta}_{\tst}$.
    \item[(2)] Generate more pairs to obtain a total of 1000 observations $\{(\bar{X}^{\tst}_i, \bar{\epsilon}_i)\}_{i=1}^{1000}$ by uniformly sampling with 
    replacement from the orignal pairs. 
    \item[(3)] For each $i$, the treatment $\bar{T}^{\tst}_i$ is selected based on the covariates $\bar{X}^{\tst}_i$ according to $\pi$. 
    \item[(4)] For each $i$, the reward $\bar{Y}^{\tst}_i$ is defined by
    \begin{equation}
        \bar{Y}^{\tst}_i = \hat{\alpha}_{\tst}^\top g(\bar{X}_i^{\tst}) + \hat{\beta}_{\tst}^\top f(\bar{X}_i^{\tst}) \bar{T}_i + \bar{\epsilon}_i,
    \end{equation}
    where the coefficients $\hat{\alpha}_{\tst}$ and $\hat{\beta}_{\tst}$ are obtained from the regression model fitted in step (1). The value is then given as the average reward: $\hat{V}(\pi) = \frac{1}{1000}\sum_{i=1}^{1000} \bar{Y}^{\tst}_i$.
\end{enumerate}
The performance of a policy $\pi$ is then computed as $\hat{V}(\pi) - \hat{V}(\pi_{0})$, where $\pi_0$ is the policy that always selects to not deliver a suggestion. This corresponds to an empirical version of the expected relative reward as in \eqref{eq:worst_case_obj}.

\begin{figure}[t]
\centering
\includegraphics[width=1.0\textwidth]{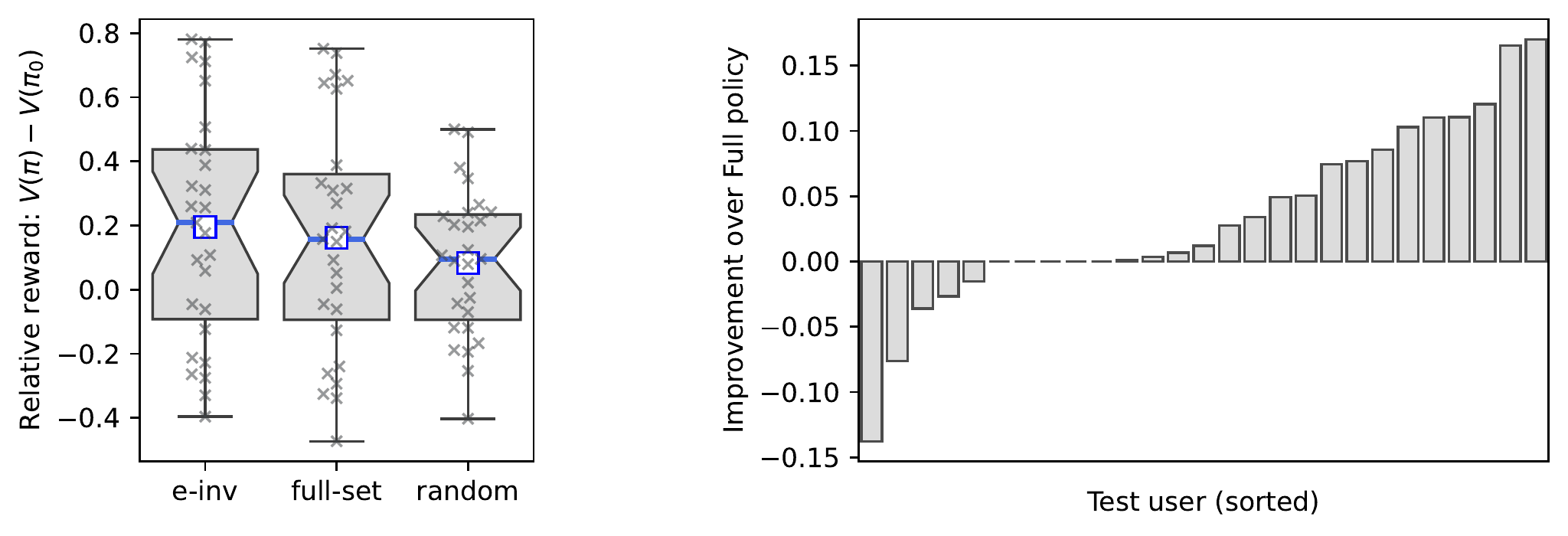}
\caption{(Left) The out-of-environment (here: out-of-user) performance of different policies including the proposed e-invariance policy (e-inv), an optimal policy that uses all variables in $f(X)$ (full-set), and a uniformly random policy (random). (Right) Comparing the out-of-environment performance between the e-inv policy and the full-set policy for each test user.
For the majority of test users, the e-inv policy outperforms the baseline (p-value 0.008).
}
\label{fig:zero-shot-simu}
\end{figure}

Figure~\ref{fig:zero-shot-simu}(Left) shows the performance of different 
policies trained on the data available during training. 
Our proposed approach (e-inv) 
shows 
a slight
improvement over the baseline approaches in terms of the mean and median performances over all users. 
Furthermore, as presented in Figure~\ref{fig:zero-shot-simu}(Right), the e-invariance policy $\pi^{\einv}$ yields higher relative reward comparing to the policy that uses all the variables in $f(X)$ in the majority of users (17 out of 27 users). 
We use the Wilcoxon signed-rank test \citep{wilcoxon_test} 
to compare the performance of the proposed e-inv policy with that of the full-set policy. It
shows a p-value of 0.008, indicating that the improvement is statistically significant.

\section{Conclusion and future work}\label{sec:conclfuturework}

This work addresses the challenge of adjusting for distribution shifts between environments in the context of policy learning. We propose an approach that leverages e-invariance, which is a relaxation of the full invariance assumption commonly used in causal inference literature. We show that despite being a weaker assumption, e-invariance is sufficient for building policies that generalize better to unseen environments compared to other policies. That is, under suitable assumptions, an optimal e-invariance policy is worst-case optimal. Additionally, we present a method for leveraging e-invariance information in the few-shot generalization setting, when a sample from the test environment is available.

To enable the practical use of e-invariance, we propose two testing procedures; one to test for e-invariance in linear and one in nonlinear model classes. Moreover, we validate the effectiveness of our policy learning methods through a semi-real-world case %
study in the domain of mobile health interventions. Our experiments show that an optimal policy based on an e-invariant set outperforms policies that rely on the complete context information when it comes to generalizing to new environments.

There are several promising directions for future research. 
It might be worthwhile to develop 
e-invariance testing procedures that can handle more complex temporal dependencies, especially when the data is collected from adaptive algorithms such as contextual bandit algorithms. Existing works have proposed inference methods to handle such scenarios \citep[e.g.,][]{zhang2021statistical, hadad2021confidence}, but how to incorporate these methods effectively into our framework remains an open question.

Another interesting area of future work is how best to use the e-invariant set $S^*$ (see \eqref{eq:opt_inv_policy_explicit}) in order to warm-start a contextual bandit algorithm. In the digital health field, one frequently conducts a series of optimization trials (each on a set of different users) in the process of optimizing a full digital health intervention. The data from each trial is used to inform the design of the subsequent trial.  In the case of HeartSteps, 3 trials (V1, V2 and V3) were conducted beginning with HeartSteps V1.  HeartSteps V2 \& V3 deployed a Bayesian Thompson-Sampling algorithm \citep{russo2018tutorial, liao2020personalized} which uses a prior distribution on the parameters to warm-start the algorithm.  Clearly the knowledge of an optimal e-invariant set $S^*$ should guide the formation of the prior. Determining the most effective approach to achieve this is still an open question.

Lastly, our work also contributes to the field of causal inference by introducing a relaxation of the full invariance assumption. We believe that there are other scenarios where the full invariance assumption is too restrictive, and a relaxation of the assumption may be sufficient to address the task at hand. 
Further investigating the potential for relaxation in different causal inference settings would be a promising future research direction.

\section*{Acknowledgments}
We thank Eura Shin for providing the code used to preprocess the HeartSteps V1 dataset. During part of this project SS and
JP were supported by a research grant (18968) from VILLUM FONDEN. NP is supported by a research grant (0069071) from Novo Nordisk Fonden. SM's research is supported by the National Institutes of Health grants P50DA054039 and P41EB028242. PK is supported by the National Institutes of Health grants R01HL125440, U01CA229445 and R01LM013107.

\bibliography{refs}

\appendix

\section{Proofs}

\subsection{Proof of Proposition~\ref{prop:p-inv_eqv}}\label{proof:prop:p-inv_eqv}

\begin{proof}
  We split the proof into three parts. First (Part 1), we show that the expected outcome function can be 
  decomposed into an effect-modification term that depends on the treatment and a main-effect term that does not 
  depend on the treatment. We then proceed and prove the `only if' part of the main result in Part 2 and the `if' part in Part 3. 
  
  \textit{Part 1:} We show the following lemma.
  \begin{lemma}\label{lemma:decompose}
   Assume Setting~\ref{setting:s1}. Let $S \subseteq \{1,\dots,d\}$ be an arbitrary subset and $t_0\in\calT$ be 
   the baseline treatment. Then, there exists a pair of functions  $\kappa_S: \calX^S \times \calT \times 
   \calE \rightarrow \R$ and $\nu_S: \calX^S \times \calE \rightarrow \R$ such that
    \begin{equation}\label{eq:lemma3}
        \forall e \in \calE, \forall x \in \calX^S,
        \forall t \in \calT: \EX^{e, \pi_t}[Y | X^S = x] = \ind(t \neq t_0)\kappa_S(x, t, e) +  \nu_S(x, e). 
    \end{equation}
  \end{lemma}
  \begin{proof}
Fix $e \in \calE$ and $t \in \calT$, and define
  $\delta_S(\cdot, t, e) \coloneqq 
  \EX^{e,\pi_t}[Y\mid X^S=\cdot]$ and $\nu_S(\cdot, e) \coloneqq \EX^{e,\pi_{t_0}}[Y\mid X^S=\cdot]$. It then 
  holds for all $x \in \calX^S$ that
  \begin{equation*}
    \EX^{e,\pi_t}[Y\mid X^S=x] = \ind(t \neq t_0)(\delta_S(x, t, e) - \nu_S(x, e)) + \nu_S(x, e).
  \end{equation*}
  We then define $\kappa_S(\cdot, t, e) \coloneqq \delta_S(\cdot, t, e) - \nu_S(\cdot, e)$, which concludes the 
  proof.
  \end{proof}

  \textit{Part 2:} Assume a subset $S \subseteq \{1,\dots,d\}$ is e-invariant \wrt $\calED$. Fix $e_0\in\calE$ as a reference environment. By 
  Lemma~\ref{lemma:decompose},
  there exists a pair of functions $\kappa_S: \calX^S \times \calT \times \calE \rightarrow \R$ and $\nu_S:
  \calX^S \times \calE \rightarrow \R$ such that for all $e\in\calED$, $x\in\calX^S$ and $t\in\calT$
  \begin{equation*}
    \tau_e^S(x,t)
    =(\ind(t \neq t_0)\kappa_S(x, t, e) + \nu_S(x, e)) - \nu_S(x, e)
    =\ind(t \neq t_0)\kappa_S(x, t, e).
  \end{equation*}
  Next, we define the function 
  $\tilde{\psi}_S: \calX^S \times \calT \rightarrow \R$ for all
  $x\in\calX^S$ and $t\in\calT$ by
  \begin{equation*}
    \tilde{\psi}_S(x,t)\coloneqq \kappa_S(x, t, e_0).
  \end{equation*}
  Now, since $S$ is e-invariant \wrt $\calED$ it
  holds for all $\forall e \in \calED$, $x\in\calX^S$ and
  $t\in\calT$ that
  \begin{equation}\label{eq:proof_prop1_inv}
    \ind(t \neq t_0)\tilde{\psi}_S(x,t) = \ind(t \neq t_0)\kappa_S(x, t, e).
  \end{equation}
  Then, combining \eqref{eq:lemma3} and \eqref{eq:proof_prop1_inv} implies that \eqref{eq:p-inv_cond2} is true.

  \textit{Part 3:} Assume \eqref{eq:p-inv_cond2} holds for a subset $S \subseteq \{1,\dots,d\}$. It then holds for all
  $e,h\in\calED$, $x\in\calX^S$ and
  $t\in\calT$ that
  \begin{align*}
    &\tau^S_e(x,t)-\tau^S_h(x,t)\\
    &\quad=(\psi_S(x,t) + \nu_S(x,e) - \psi_S(x,t_0) - \nu_S(x,e))-
      (\psi_S(x,t) + \nu_S(x,h) - \psi_S(x,t_0)- \nu_S(x,h))\\
    &\quad=0,
  \end{align*}
  which proves that $S$ is e-invariant \wrt
  $\calED$.
\end{proof}

\subsection{Proof of Proposition~\ref{prop:suffforS}}\label{proof:prop:suffforS}
From the SCM \eqref{eq:setting2-scm}, we have for all $x \in \calX^{\PA_{f,X}}$  that
\begin{align*}
    \EX^{e, \pi_t}[Y \mid X^{\PA_{f,X}}=x] &= \EX^{e}[f(X^{\PA_{f,X}}, U^{\PA_{f,U}}, t) \mid X^{\PA_{f,X}}=x] + \EX^{e}[g_e(X, U, \epsilon_Y) \mid X^{\PA_{f,X}}=x]
\intertext{Using the assumption (i) that $U^{\PA_{f,U}} \ci X^{\PA_{f,X}}$ in $\P^{e}_{X,U}$ for all $e \in \calED$, we have}
    \EX^{e, \pi_t}[Y \mid X^{\PA_{f,X}}=x] &=\EX^{e}[f(x, U^{\PA_{f,U}}, t)] + \EX^{e}[g_e(X, U, \epsilon_Y) \mid X^{\PA_{f,X}}=x], \numberthis \label{eq:proof_prop6_1}
\end{align*}
where a formal proof for the equality \eqref{eq:proof_prop6_1} is given, for example, in \cite[Example~4.1.7]{durrett2019probability}.
Next, using the assumption (ii) that $\P^e_{U^{\PA_{f,U}}}$ are identical across $e \in \calED$, we can drop the dependency on $e$ from the component $\EX^{e}[f(x, U^{\PA_{f,U}}, t)]$ in \eqref{eq:proof_prop6_1} and have that
\begin{equation*}
    \EX^{e, \pi_t}[Y \mid X^{\PA_{f,X}}=x] =\EX[f(x, U^{\PA_{f,U}}, t)] + \EX^{e}[g_e(X, U, \epsilon_Y) \mid X^{\PA_{f,X}}=x].
\end{equation*}
Thus, by Proposition~\ref{prop:p-inv_eqv}, $\PA_{f,X}$ is e-invariant \wrt $\calED$.

\subsection{Proof of Proposition~\ref{prop:gen_policy_ident}}\label{proof:prop:gen_policy_ident}
\begin{proof}
We begin with the proof of the first statement (i) of Proposition~\ref{prop:gen_policy_ident}. First, we show that the collections of policies $(\Pi^S_{\opt})_{S \in \mS^{\einv}_{\calEtr}}$ are
identifiable from $Q^{\tr}_i$ (for an arbitrary $1\leq i \leq n$).
Fix (an arbitrary) $f \in \calEtr$ and $S \in 
\mS^{\einv}_{\calEtr}$. Let $\pi^S \in \Pi^S_{\opt}$. Then, $\pi^S$ satisfies
\begin{equation}
    \pi^S(t | x) > 0 \implies 
    t  \in \argmax_{t^\prime \in \calT} \tfrac{1}{\abs{\calEtr}} \sum_{e \in \calEtr} \tau^S_e(x^S, t^\prime) = \argmax_{t^\prime \in \calT} \tau^S_f(x^S, t^\prime).
\end{equation}
Thus, the identifiability of $\Pi^S_{\opt}$ depends on the identifiability of $\tau^S_f$. 

Fix $i \in \{1,\dots,n\}$. 
Recall that in Setting~\ref{setting:s1} we assume $\forall x \in \calX, t \in \calT: \pi_i(t \mid x)
> 0$. 
It then holds for all $x \in \calX^S$ and $t \in \calT$ that 
\begin{align*}
    \EX^{e_i, \pi_t}[Y_i \mid X^S_i = x] 
    &= \EX^{e_i}[\EX^{e_i, \pi_t}[Y_i \mid X_i] \mid X^S_i = x] \\
    &\overset{(*)}{=}\EX^{e_i}[\EX^{e_i, \pi_i}\big[\tfrac{\ind(T_i=t)}{\pi_i(T_i \mid X_i)} Y_i \mid X_i \big] \mid X^S_i=x] \\
    &=\EX^{e_i, \pi_i}\big[\tfrac{\ind(T_i=t)}{\pi_i(t \mid X_i)} Y_i \mid X^S_i =x \big], \numberthis \label{eq:proof_propo3_2}
\end{align*}
where the equality $(*)$ holds by definition of $\pi_t$.
Since the right-hand side of \eqref{eq:proof_propo3_2} is the expectation \wrt $Q^{\tr}_i$, the quantity $\EX^{e_i, \pi_t}[Y_i \mid X^S_i = x]$ is identifiable from $Q^{\tr}_i$.

Next, we have
\begin{align*}
    \EX^{e_i, \pi_t}[Y_i \mid X^S_i = x] - \EX^{e_i, \pi_0}[Y_i \mid X^S_i = x] &= \EX^{e_i, \pi_t}[Y \mid X^S = x] - \EX^{e_i, \pi_0}[Y \mid X^S = x] \\
     &= \tau^S_{e_i}(x, t) \\
     &= \tau^S_{f}(x, t) \qquad \qquad \text{since $S \in \mS^{\einv}_{\calEtr}$.} \numberthis \label{eq:proof_propo3_4}
\end{align*}

From \eqref{eq:proof_propo3_4}, we then have that $\tau^S_f$ is identifiable from
$Q^{\tr}_i$ and therefore $\Pi^S_{\opt}$ is identifiable from $Q^{\tr}_i$. Consequently, the collection $$
A \coloneqq \argmax_{S \in \mS^{\einv}_{\calEtr}}\EX^{e^{\tst}} \left[ \textstyle\sum_{t\in\calT}\tau^S_{e}(X^S,t) \pi^S(t \mid X) \right]$$ is identifiable from $Q^{\tr}_i$ and $Q^{\tst}$.

Next, we show the proof of the second statement (ii) of Proposition~\ref{prop:gen_policy_ident}.
Fix $e^{\tst} \in \calE$, (an arbitrary) $f \in \calEtr$ and let $S^* \subseteq \{1,\dots,d\}$ be 
 a subset that satisfies
\begin{equation*}
    S^* \in \argmax_{S \in \mS^{\einv}_{\calEtr}}\EX^{e^{\tst}}[\sum_{t\in\calT}\tau^S_{f}(X^S,t) \pi^S(t \mid X^S)].
\end{equation*}
Next, we recall the defnition $\Pi^{\einv}_{\opt} = \{\pi \in \Pi \mid \exists S \in \mS^{\einv}_{\calEtr}\ 
\text{s.t.\ } \pi \in \Pi^S_{\opt}\}$ and let $\pi^{\diamond}_{e^{\tst}} \in \argmax_{\pi \in \Pi^{\einv}_{\opt}} \EX^{e^{\tst}, \pi}[Y]$. Then, using that $\pi^{\diamond}_{e^{\tst}} \in\Pi^{\einv}_{\opt}$, choose $S^{\diamond} \in \mS^{\einv}_{\calEtr}$ such that
$\pi^{\diamond}_{e^{\tst}}\in\Pi^{S^{\diamond}}_{\opt}$.
We have
\begin{align*}
    \EX^{e^{\tst}, \pi^{\diamond}_{e^{\tst}}}[Y] - \EX^{e^{\tst}, \pi_{t_0}}[Y] &=  \EX^{e^{\tst}}[\sum_{t\in\calT}(\EX^{e^{\tst},\pi_t}[Y \mid X^{S^{\diamond}}] - \EX^{e^{\tst},\pi_{t_0}}[Y \mid X^{S^{\diamond}}])\pi^{\diamond}_{e^{\tst}}(t \mid X^{S^{\diamond}})] \\
    &= \EX^{e^{\tst}}[\sum_{t\in\calT}\tau^{S^\diamond}_{e^{\tst}}(X^{S^\diamond},t) \pi^{\diamond}_{e^{\tst}}(t \mid X^{S^\diamond})] \\
    &= 
    \EX^{e^{\tst}}[\sum_{t\in\calT}\tau^{S^\diamond}_{f}(X^{S^\diamond},t)  \pi^{\diamond}_{e^{\tst}}(t \mid X^{S^\diamond})] \qquad \qquad \text{by Assumption~\ref{assm:2}}\\
    &= 
    \EX^{e^{\tst}}[\sum_{t\in\calT}\tau^{S^\diamond}_{f}(X^{S^\diamond},t) \pi^{S^\diamond}(t \mid X^{S^\diamond})] \qquad \qquad \text{since $\pi^{\diamond}_{e^{\tst}} \in \Pi^{S^{\diamond}}_{\opt}$} \\
    &\leq \max_{S \in \mS^{\einv}_{\calEtr}} \EX^{e^{\tst}}[\sum_{t\in\calT}\tau^S_{f}(X^S,t) \pi^S(t \mid X^S)] 
    \\
    &= \EX^{e^{\tst}}[\sum_{t\in\calT}\tau^{S^*}_{f}(X^{S^*},t) \pi^{S^*}(t \mid X^{S^*})] \\
    &= \EX^{e^{\tst}}[\sum_{t\in\calT}\tau^{S^*}_{e^{\tst}}(X^{S^*},t) \pi^{S^*}(t \mid X^{S^*})] \qquad \quad \text{by Assumption~\ref{assm:2}} \\
    &= \EX^{e^{\tst}}[\sum_{t\in\calT}(\EX^{e^{\tst},\pi_t}[Y \mid X^{S^*}] - \EX^{e^{\tst},\pi_{t_0}}[Y \mid X^{S^*}])\pi^{S^*}(t \mid X^{S^*})] \\
    &= \EX^{e^{\tst}, \pi^{S^*}}[Y] - \EX^{e^{\tst}, \pi_{t_0}}[Y].
\end{align*}
We therefore have that
\begin{equation*}
    \EX^{e^{\tst}, \pi^{\diamond}_{e^{\tst}}}[Y] \leq \EX^{e^{\tst}, \pi^{S^*}}[Y],
\end{equation*}
which concludes the proof of the second statement (ii) of Proposition~\ref{prop:gen_policy_ident}.

\end{proof}

\subsection{Proof of Theorem~\ref{thm:inv_policy}}\label{proof:thm:inv_policy}
\begin{proof}
Let $e^{\tst} \in \calE$ be a test environment, 
$\pi^{\einv}$ be a policy 
satisfying \eqref{eq:opt_inv_policy} and, for all $S \subseteq \{1,\dots,d\}$, $\Pi^S_{\opt}$ 
be the set policies satisfying \eqref{eq:pi_S}. 

We now prove the first statement, see Theorem~\ref{thm:inv_policy_1}. By definition, there exists $S^* \in \mS^{\einv}_{\calEtr}$ such that $\pi^{\einv} \in \Pi^{S^*}_{\opt}$.  It then holds that
\begin{align*}
    \EX^{e^{\tst}, \pi^{\einv}}[Y] - \EX^{e^{\tst}, \pi_{t_0}}[Y] &= \EX^{e^{\tst}}[\sum_{t\in\calT}(\EX^{e^{\tst},\pi_t}[Y \mid X^{S^*}] - \EX^{e^{\tst},\pi_{t_0}}[Y \mid X^{S^*}])\pi^{\einv}(t \mid X^{S^*})] \\
    &= \EX^{e^{\tst}}[\sum_{t\in\calT}\tau^{S^*}_{e^{\tst}}(X^{S^*}, t)\pi^{\einv}(t \mid X^{S^*})].
\end{align*}
Fix $e^{\tr} \in \calEtr$. We have
\begin{align*}
    \EX^{e^{\tst}, \pi^{\einv}}[Y] - \EX^{e^{\tst}, \pi_{t_0}}[Y] &= \EX^{e^{\tst}}[\sum_{t\in\calT}\tau^{S^*}_{e^{\tst}}(X^{S^*}, t)\pi^{\einv}(t \mid X^{S^*})] \\
    &= \EX^{e^{\tst}}[\sum_{t\in\calT} \tau^{S^*}_{e^{\tr}}(x^{S^*}, t) \pi^{\einv}(t \mid X^{S^*})] && \text{by Assumption~\ref{assm:2}} \\
    &= \EX^{e^{\tst}}[\max_{t\in\calT} \tau^{S^*}_{e^{\tr}}(x^{S^*}, t)] && \text{by the definition of $\Pi^{S^*}_{\opt}$} \\
    &= \EX^{e^{\tst}}[\max_{t\in\calT}\tau^{S^*}_{e^{\tst}}(X^{S^*}, t)] && \text{by Assumption~\ref{assm:2}} \\
    &\geq \max_{t \in \calT}(\EX^{e^{\tst}}[\tau^{S^*}_{e^{\tst}}(X^{S^*}, t)]) \\
    &= \max_{t \in \calT}(\EX^{e^{\tst}, \pi_t}[Y] - \EX^{e^{\tst},\pi_{t_0}}[Y]) && \text{by the tower property}
\end{align*}
This implies, 
\begin{equation}
    \EX^{e^{\tst}, \pi^{\einv}}[Y] \geq \max_{t \in \calT}\EX^{e^{\tst}, \pi_t}[Y],
\end{equation}
which concludes the proof of Theorem~\ref{thm:inv_policy_1}.
Next, we prove the second statement, see Theorem~\ref{thm:inv_policy_2}. Recall that $[e^{\tst}] \coloneqq \{ e\in\calE \mid \P^e_X = Q^{\tst}_X \}$.
From Assumption~\ref{assm:3}, there exists an environment $f \in [e^{\tst}]$ and $\SD \in \mS^{\einv}_{\calE}$ such that 
\begin{equation} \label{eq:proof_thm1_1}
   \forall x \in \calX: \max_{t \in \calT} \tau_f(x, t) = \max_{t \in \calT} \tau^{\SD}_f(x^{\SD}, t).
\end{equation}

We have for all $S \in \mS^{\einv}_{\calEtr}$ that 
\begin{align*} 
    \EX^f[\max_{t \in \calT} \tau^{\SD}_f(X^{\SD}, t)] &= \EX^f[\max_{t \in \calT} \tau_f(X, t)] && \text{from \eqref{eq:proof_thm1_1}} \\
    &= \EX^f[\EX^f[\max_{t \in \calT} \tau_f(X ,t)\mid X^S]]  \\
    &\geq \EX^f[\max_{t \in \calT}\EX^f[\tau_f(X ,t) \mid X^S]] \\
    &= \EX^f[\max_{t \in \calT}\tau^S_f(X^S ,t)]. \numberthis \label{eq:eq:proof_thm1_1-1}
\end{align*}

Now, we have for all $e \in [e^{\tst}]$ and for all $S \in \mS^{\einv}_{\calEtr}$
\begin{align*}
    \EX^e[\max_{t \in \calT} \tau^{\SD}_e(X^{\SD}, t)] &= \EX^f[\max_{t \in \calT} \tau^{\SD}_f(X^{\SD}, t)] && \text{by Assumption~\ref{assm:2} and $f \in [e^{\tst}]$} \\
    &\geq \EX^f[\max_{t \in \calT}\tau^S_f(X^S ,t)] && \text{from \eqref{eq:eq:proof_thm1_1-1}} \\
    &= \EX^e[\max_{t \in \calT}\tau^S_e(X^S ,t)] && \text{by Assumption~\ref{assm:2} and $f \in [e^{\tst}]$.} \numberthis \label{eq:proof_thm1_2}
\end{align*}

Next, we recall the definition $\Pi^{\einv}_{\opt} = \{\pi \in \Pi \mid \exists S \in \mS^{\einv}_{\calEtr}\ 
\text{s.t.\ } \pi \in \Pi^S_{\opt}\}$ and that $\pi^{\einv} \in \argmax_{\pi\in \Pi^{\einv}_{\opt}} \EX^{e^{\tst}, \pi}[Y]$. Then there exists $S^* \in \mS^{\einv}_{\calEtr}$ such that $\pi^{\einv} \in \Pi^{S^*}_{\opt}$. We therefore have for all $e \in [e^{\tst}]$ that
\begin{align*}
    \EX^{e, \pi^{\einv}}[Y] - \EX^{e, \pi_{t_0}}[Y] &= \EX^{e}[\sum_{t\in\calT}(\EX^{e,\pi_t}[Y \mid X^{S^*}] - \EX^{e,\pi_{t_0}}[Y \mid X^{S^*}])\pi^{\einv}(t \mid X^{S^*})] \\
    &= \EX^{e}[\sum_{t\in\calT}\tau^{S^*}_{e}(X^{S^*}, t)\pi^{\einv}(t \mid X^{S^*})]. \numberthis \label{eq:proof_thm1_2-2}
\end{align*}
Fix $e^{\tr} \in \calEtr$, we have for all $e \in [e^{\tst}]$ 
\begin{align*}
    \EX^{e, \pi^{\einv}}[Y] - \EX^{e, \pi_{t_0}}[Y] &= \EX^{e}[\sum_{t\in\calT}\tau^{S^*}_{e^{\tr}}(X^{S^*}, t)\pi^{\einv}(t \mid X^{S^*})] \qquad \text{by \eqref{eq:proof_thm1_2-2} and Assumption~\ref{assm:2}} \\
    &= \EX^{e}[\max_{t \in \calT} \tau^{S^*}_{e^{\tr}}(X^{S^*}, t)] \qquad \qquad \text{by the definition of $\Pi^{S^*}_{\opt}$} \\
    &= \EX^{e}[\max_{t \in \calT} \tau^{S^*}_{e}(X^{S^*}, t)] \qquad \qquad \text{by Assumption~\ref{assm:2}.} \numberthis \label{eq:proof_thm1_3}
\end{align*}
Let $\pi^{S^{\diamond}} \in \Pi^{S^{\diamond}}_{\opt}$, we then have that
\begin{align*}
    \EX^{e^{\tst}}[\max_{t \in \calT} \tau^{S^*}_{e^{\tst}}(X^{S^*}, t)] 
    &= 
    \EX^{e^{\tst}, \pi^{\einv}}[Y] - \EX^{e^{\tst}, \pi_{t_0}}[Y] 
    && \text{from \eqref{eq:proof_thm1_3}} 
    \\
    &\geq \EX^{e^{\tst}, \pi^{S^{\diamond}}}[Y] - \EX^{e^{\tst}, \pi_{t_0}}[Y] \\
    &= \EX^{e^{\tst}}[\sum_{t\in\calT}(\EX^{e^{\tst},\pi_t}[Y \mid X^{S^{\diamond}}] - \EX^{e^{\tst},\pi_{t_0}}[Y \mid X^{S^{\diamond}}])\pi^{S^{\diamond}}(t \mid X^{S^{\diamond}})] \\
    &= \EX^{e^{\tst}}[\sum_{t\in\calT}\tau^{S^{\diamond}}_{e^{\tst}}(X^{S^{\diamond}}, t)\pi^{S^{\diamond}}(t \mid X^{S^{\diamond}})] \\
    &= \EX^{e^{\tst}}[\sum_{t\in\calT}\tau^{S^{\diamond}}_{e^{\tr}}(X^{S^{\diamond}}, t)\pi^{S^{\diamond}}(t \mid X^{S^{\diamond}})] \qquad \text{by Assumption~\ref{assm:2}} \\
    &= \EX^{e^{\tst}}[\max_{t \in \calT} \tau^{S^{\diamond}}_{e^{\tr}}(X^{S^{\diamond}}, t)] \qquad \qquad \text{by the definition of $\Pi^{S^{\diamond}}_{\opt}$} \\
    &= \EX^{e^{\tst}}[\max_{t \in \calT} \tau^{S^{\diamond}}_{e^{\tst}}(X^{S^{\diamond}}, t)], \qquad \qquad \text{by Assumption~\ref{assm:2}} \numberthis \label{eq:proof_thm1_3-2}
\end{align*}
where the above inequality holds because 
$$\pi^{\einv} \in \argmax_{\pi\in \Pi^{\einv}_{\opt}} \EX^{e^{\tst}, \pi}[Y] = \argmax_{\pi\in \Pi^{\einv}_{\opt}} \EX^{e^{\tst}, \pi}[Y] - \EX^{e^{\tst}, \pi_{t_0}}[Y].$$ 
Combining the two inequalities \eqref{eq:proof_thm1_3-2} and \eqref{eq:proof_thm1_2}, 
we then have that
\begin{equation}\label{eq:proof_thm1_3-3}
    \EX^{e^{\tst}}[\max_{t \in \calT} \tau^{S^*}_{e^{\tst}}(X^{S^*}, t)] = \EX^{e^{\tst}}[\max_{t \in \calT} \tau^{S^{\diamond}}_{e^{\tst}}(X^{S^{\diamond}}, t)]. 
\end{equation}

Then, from \eqref{eq:proof_thm1_3-3} and since $e^{\tst} \in [e^{\tst}]$ and $S^{\diamond}$ is e-invariant \wrt $\calEtr \cup [e^{\tst}]$ (by Assumption~\ref{assm:2}), we have for all $e \in [e^{\tst}]$
\begin{equation}\label{eq:proof_thm1_4}
    \EX^{e}[\max_{t \in \calT} \tau^{S^*}_{e}(X^{S^*}, t)] = \EX^{e}[\max_{t \in \calT} \tau^{S^{\diamond}}_{e}(X^{S^{\diamond}}, t)].
\end{equation}

We are now ready to prove the main statement of Theorem~\ref{thm:inv_policy_2}.
\begin{align*}
    V^{[e^{\tst}]}(\pi^{\einv}) &= \inf_{e \in [e^{\tst}]}(\EX^{e, \pi^{\einv}}[Y] - \EX^{e, \pi_{t_0}}[Y]) \\ 
    &= \inf_{e \in [e^{\tst}]} \EX^{e}[\max_{t\in\calT}\tau^{S^*}_e(X^{S^*}, t)] && \text{from \eqref{eq:proof_thm1_3}.}
\intertext{By the definition of $[e^{\tst}]$ and Assumption~\ref{assm:2}, we then have}
V^{[e^{\tst}]}(\pi^{\einv}) &= \EX^{f}[\max_{t\in\calT}\tau^{S^*}_{f}(X^{S^*}, t)] \\
&= \EX^{e^f}[\max_{t \in \calT} \tau^{S^{\diamond}}_{e^f}(X^{S^{\diamond}}, t)] && \text{from \eqref{eq:proof_thm1_4}.} \numberthis \label{eq:proof_thm1_5}
\end{align*}

Finally, we show that for all $\pi \in \Pi$, $V^{[e^{\tst}]}(\pi)$ is bounded above by $V^{[e^{\tst}]}(\pi^{\einv})$.
\begin{align*}
    V^{[e^{\tst}]}(\pi) &= \inf_{e \in [e^{\tst}]} (\EX^{e^{\tst}, \pi}[Y] - \EX^{e^{\tst}, \pi_{t_0}}[Y]) \\ &\leq \EX^{f, \pi}[Y] - \EX^{f, \pi_{t_0}}[Y] && \text{since $f \in [e^{\tst}]$}\\
    &= \EX^f[\sum_{t\in\calT}(\EX^{f,\pi_t}[Y \mid X] - \EX^{f,\pi_{t_0}}[Y \mid X])\pi(t \mid X)] \\
    &= \EX^f[\sum_{t\in\calT}\tau_f(X, t)\pi(t \mid X)] \\
    &\leq \EX^f[\max_{t \in \calT} \tau_f(X, t)] \\
    &= \EX^f[\max_{t \in \calT} \tau^{S^{\diamond}}_f(X^{\SD}, t)] && \text{from \eqref{eq:proof_thm1_1}} \\
    &= V^{[e^{\tst}]}(\pi^{\einv}) && \text{from \eqref{eq:proof_thm1_5}},
\end{align*}
which concludes the proof.
\end{proof}

\subsection{Proof of Proposition~\ref{prop:wald_test}}\label{proof:prop:wald_test}
\begin{proof}
Let $S \subseteq \{1,\dots,d\}$, $\alpha \in (0,1)$, $\tilde{X} \coloneqq \rvector{1 & X}^\top$ and $\tilde{X}^S \coloneqq \rvector{1 & X^S}^\top$ and $\hat{B}$ be the estimator 
solving the equation
$\sum_{i}^n G_i(\alpha, A, \beta,B) = 0$. Assume Setting~\ref{setting:s1} and assume the following regularity conditions (these are similar to the ones required by \citet{boruvka2018assessing} with the difference that we require them to hold for all $e \in \calEtr$).
\begin{assumption}[Regularity conditions]\label{assm:regularity_cond} For all $e \in \calEtr$ it holds that
\begin{itemize}
    \item[(i)] 
    $\EX^{e, \pi^{\tr}}[Y^4]<\infty$ and  $\max_{j\in\{1,\ldots,d+1\}}\EX^{e}[(\tilde{X}^j)^4]<\infty$,
    \item[(ii)] the matrices $\EX^e[\tilde{X}^S(\tilde{X}^S)^\top]$ 
    and 
    \begin{equation*}
        \EX^e\left[\sum_{t \in \calT} \tilde{\pi}(t |X^S)\begin{bmatrix}
    \tilde{X} \\
    u_{e}\tilde{X} \\
    v_{t} - \tilde{\pi}(1|X^S) \\
    (v_{t} - \tilde{\pi}(1|X^S))u_e
\end{bmatrix}\begin{bmatrix}
    \tilde{X} \\
    u_{e}\tilde{X} \\
    v_{t} - \tilde{\pi}(1|X^S) \\
    (v_{t} - \tilde{\pi}(1|X^S))u_e
\end{bmatrix}^\top\right]
    \end{equation*}
are invertible.
\end{itemize}
\end{assumption}
From Proposition~3.1 in \cite{boruvka2018assessing}, we have that $\sqrt{n}(\hat{B} - B) \indist \mathcal{N}(0, \VAR[B])$, where $\hat{\VAR}$ defined in \eqref{eq:covariance_estimator} is a consistent estimator of $\VAR[B]$. 
Next, from Theorem~8.3 in \cite{boos2013essential}, we have that $T_n \coloneqq n\hat{B}\hat{\VAR}\hat{B} \indist \chi^2_{\abs{\vec(B)}}$. Let  $q_{\alpha}$ be the $(1 - \alpha)$-quantile of $\chi^2_{\abs{\vec(B)}}$ and $\psi^{\text{Wd}}_n(D^{\tr}, S, \alpha) := \ind(T_n > q_{\alpha})$. We can then conclude that
\begin{equation*}
    \sup_{\P \in H^{\tr}_{0, S}}\limsup_{n \to \infty} \P(\psi^{\text{Wd}}_n(D^{\tr}, S, \alpha) = 1) \leq \alpha.
\end{equation*}

\end{proof}

\subsection{Proof of Proposition~\ref{prop:dr_test1}}\label{proof:prop:dr_test1}
\begin{proof}
For all $x \in \calX, t \in \calT, w \in \calT$ and $e \in \calE$, let $Z^e_w(x, t, y) \coloneqq \bar{\mu}^e(x, w) + \frac{\ind(t = w)(y - \bar{\mu}^e(x, w))}{\bar{\pi}(w | x)}$. We now show that for all $S \subseteq \{1,\dots,d\}$, $\EX^{e, \pi^{\tr}}[Z^e_w(X, T, Y) \mid X^S] = \EX^{e,\pi_w}[Y \mid X^S]$ if one of the models is correct.
\begin{enumerate}[label=(\roman*)]
    \item Assume that $\bar{\mu}^e$ is correct, i.e., $\bar{\mu}^e(x, t) = \EX^{e}[Y \mid X=x, T=t]$ for all $x \in \calX$, $t \in \calT$ and $e \in \calEtr$. We then have 
    \begin{align*}
    \EX^{e,\pi^{\tr}}\big[\frac{\ind(T = w)(Y - \bar{\mu}^e(X, w))}{\bar{\pi}(w \mid X)} \mid X^S\big] &= \EX^e\big[\EX^{e,\pi^{\tr}}\big[\frac{\ind(T = w)(Y - \bar{\mu}^e(X, w))}{\bar{\pi}(w | X)} \mid X \big] \mid X^S \big] \\
    &=\EX^e\big[\frac{\pi^{\tr}(w | X)}{\bar{\pi}(w | X)}\EX^{e}\big[(Y - \bar{\mu}^e(X, w)) \mid X, T = w \big] \mid X^S \big] \\
    &= 0. \numberthis \label{eq:proof_prop11_1}
    \end{align*}
    Next, we have
    \begin{align*}
        \EX^{e, \pi^{\tr}}\big[\bar{\mu}^e(X,w) \mid X^S\big] &= \EX^{e}\big[\EX^{e}\big[Y \mid X, T = w \big]\mid X^S\big] \\
        &= \EX^{e}\big[\EX^{e, \pi_w}\big[Y \mid X \big]\mid X^S\big] \\
        &= \EX^{e, \pi_w}\big[Y \mid X^S\big].
        \numberthis \label{eq:proof_prop6_2}
    \end{align*}
    Then, from \eqref{eq:proof_prop11_1} and \eqref{eq:proof_prop6_2}, we have that $\EX^{e, \pi^{\tr}}[Z^e_w(X, T, Y) \mid X^S] = \EX^{e,\pi_w}[Y \mid X^S]$ and it thus holds for all $x \in \calX^S$ that 
    \begin{equation*}
        \EX^{e,\pi^{\tr}}[O^e(X,T,Y) \mid X^S = x] = \EX^{e, \pi^{\tr}}[Z^e_1(X, T, Y) - Z^e_0(X, T, Y) \mid X^S = x]  = \tau^S_e(x, 1).
    \end{equation*}
    \item Assume that $\bar{\pi}$ is correct, i.e., $\bar{\pi}(t \mid x) = \pi^{\tr}(t \mid x)$ for all $x \in \calX$, $t \in \calT$. We then have
    \begin{align*}
    \EX^{e,\pi^{\tr}}\big[\frac{\ind(T = w)(Y - \bar{\mu}^e(X, w))}{\bar{\pi}(w \mid X)} \mid X^S\big] &= \EX^e\big[\EX^{e,\pi^{\tr}}\big[\frac{\ind(T = w)(Y - \bar{\mu}^e(X, w))}{\pi^{\tr}(w | X)} \mid X\big] \mid X^S \big] \\
    &= \EX^e\big[\EX^{e}\big[(Y - \bar{\mu}^e(X, w)) \mid X, T = w\big] \mid X^S \big] \\
    &= \EX^e\big[\EX^{e,\pi_w}\big[Y - \bar{\mu}^e(X, w) \mid X\big] \mid X^S \big] \\
    &= \EX^{e,\pi_w}\big[Y - \bar{\mu}^e(X, w) \mid X^S\big]. \numberthis \label{eq:proof_prop6_3}
    \end{align*}
Next, we have
\begin{align*}
    \EX^{e, \pi^{\tr}}\big[\bar{\mu}^e(X,w) \mid X^S\big] &= \EX^{e, \pi_w}\big[\bar{\mu}^e(X,w) \mid X^S\big], \numberthis \label{eq:proof_prop6_4}
\end{align*}
since the expectation is only over $X$ and does not depend on the treatment $T$. Then, from \eqref{eq:proof_prop6_3} and \eqref{eq:proof_prop6_4}, we have that $\EX^{e, \pi^{\tr}}[Z^e_w(X, T, Y) \mid X^S] = \EX^{e,\pi_w}[Y \mid X^S]$ and it thus holds for all $x \in \calX^S$ that 
\begin{equation*}
    \EX^{e,\pi^{\tr}}[O^e(X,T,Y) \mid X^S = x] = \EX^{e, \pi^{\tr}}[Z^e_1(X, T, Y) - Z^e_0(X, T, Y) \mid X^S = x]  = \tau^S_e(x, 1).
\end{equation*}

\end{enumerate}

\end{proof}

\section{Few-shot policy generalization in linear models} 
\label{app:fewshotlinear}

\begin{example}[Few-shot policy generalization for linear CATE functions]
Let $S \in \mS^{\einv}_{\calEtr}$
and $N\coloneqq\{1,\ldots,d\}\setminus S$ and recall that $\calT = \{1,\dots,k\}$. We assume that $\calH$ is a class of 
linear functions parameterized by $\Theta \subseteq \R^{k \times d}$, i.e., $\calH \coloneqq \{\tau \mid \exists \theta
\in \Theta  \text{ s.t. } \forall t \in \calT: \tau(\cdot, t) \equiv \theta_t (\cdot)\}$.

By the linearity of $\calH$, Assumption~\ref{assm:few-shot} is satisfied, that is, there exists $\theta^S \in \Theta^S \subseteq \R^{k \times \abs{S}}$ and $\theta^N \in \Theta^N \subseteq \R^{k \times \abs{N}}$ such that
\begin{equation}
    \forall x \in \calX, t \in \calT: \tau_{e^{\tst}}(x, t) = \theta_t^S x^S + \theta_t^N x^N.
\end{equation}
Under Assumption~\ref{assm:2}, we then have
\begin{equation}
    \forall x \in \calX^S, t \in \calT: \theta_t^S x = \tau^S_{\tr}(x, t) - \theta_t^N \EX^{e^{\tst}}[X^N \mid X^S = x],
\end{equation}
and hence
\begin{equation}
    \forall x \in \calX, t \in \calT: \tau_{e^{\tst}}(x, t) = \tau^S_{\tr}(x^S, t) + \theta_t^N (x^N - \EX^{e^{\tst}}[X^N \mid X^S = x^S]).
\end{equation}
Next, let $\hat{q}$ be an estimator of
$\EX^{e^{\tst}}[X^N \mid X^S = \cdot]$ and, for all $\theta \in \Theta^N$, define $\tau_{\theta}^S: (x ,t) \mapsto \tau^S_{\tr}(x^S, t) + \theta_t(x^N - \hat{q}(x^S))$. Importantly, the estimand we are estimating by $\hat{q}$ does not change with $\theta^N$. 
We then consider the unconstrained optimization
\begin{equation}\label{eq:unconstrained_opt_lin}
    \hat{\theta} \in \argmin_{\theta \in \Theta^N} \sum_{i=1}^m \ell(Y^{\tst}_i, X^{\tst}_i, T^{\tst}_i, \pi^{\tst}_i, \tau_{\theta}^S).
\end{equation}
Here, by utilizing the e-invariance information $\tau^S_{tr}$ along with Assumption~\ref{assm:2}, we now optimize over the restricted function class $\Theta^N \subset \Theta$.
\end{example}

\end{document}